\documentclass[twoside,11pt]{article}

\usepackage[preprint]{jmlr2e}
\usepackage[utf8]{inputenc}
\usepackage{mathtools}
\usepackage{multirow}
\usepackage{booktabs}
\usepackage{lscape} 
\usepackage{longtable}
\usepackage{algpseudocode, algorithm}

\usepackage{amsmath,amssymb,amsfonts,mathrsfs,sansmath,dsfont}

\newcommand{\argmax}{\arg\,\max}
\newcommand{\fc}{f_\mathcal{C}}
\newcommand{\fo}{f_\mathcal{O}}

\jmlrheading{1}{2022}{1-48}{4/00}{10/00}{meila00a}{Nahuel Statuto, Irene Unceta, Jordi Nin and Oriol Pujol}

\ShortHeadings{Sequential Copies}{Statuto et al.}
\firstpageno{1}

\begin{document}

\title{A Scalable and Efficient Iterative Method for Copying Machine Learning Classifiers}

\author{\name Nahuel Statuto \email nahuel.statuto@esade.edu \\
       \addr Department of Operations, Innovation and Data Sciences\\
       Universitat Ramon Llull, ESADE\\
       Sant Cugat del Vallès, 08172, Catalonia, Spain
       \AND
       \name Irene Unceta \email irene.unceta@esade.edu \\
       \addr Department of Operations, Innovation and Data Sciences\\
       Universitat Ramon Llull, ESADE\\
       Sant Cugat del Vallès, 08172, Catalonia, Spain
       \AND
        \name Jordi Nin \email jordi.nin@esade.edu \\
       \addr Department of Operations, Innovation and Data Sciences\\
       Universitat Ramon Llull, ESADE\\
       Sant Cugat del Vallès, 08172, Catalonia, Spain
       \AND
       \name Oriol Pujol \email oriol\_pujol@ub.edu \\
       \addr Departament de Matem\`atiques i Inform\`atica\\
       Universitat de Barcelona\\
       Barcelona, 08007, Catalonia, Spain}

\editor{}

\maketitle

\begin{abstract}%   <- trailing '%' for backward compatibility of .sty file

Differential replication through copying refers to the process of replicating the decision behavior of a machine learning model using another model that possesses enhanced features and attributes. This process is relevant when external constraints limit the performance of an industrial predictive system. Under such circumstances, copying enables the retention of original prediction capabilities while adapting to new demands. Previous research has focused on the single-pass implementation for copying. This paper introduces a novel sequential approach that significantly reduces the amount of computational resources needed to train or maintain a copy, leading to reduced maintenance costs for companies using machine learning models in production. The effectiveness of the sequential approach is demonstrated through experiments with synthetic and real-world datasets, showing significant reductions in time and resources, while maintaining or improving accuracy.
\end{abstract}


\section{Introduction}
\label{sec:introduction}
Machine learning has become widespread in many industries, with supervised algorithms automating tasks with higher precision and lower costs~\citep{GOMEZ2018175, 10.1145/3446662, SHEHAB2022105458, doi:10.1080/13675567.2020.1803246}. However, maintaining the performance of industrial machine learning models requires constant monitoring as the environment where they are deployed can change due to internal and external factors, such as new production needs, technological updates, novel market trends, or regulatory changes. Neglecting these changes can lead to model degradation and decreased performance. To prevent this, regular model monitoring is essential in any industrial machine learning application. Once served into production, models are frequently checked for any signs of performance deviation, which can occur just a few months after deployment. In case of deviation, models are either fully or partially retrained and substituted~\citep{pmlr-v119-wu20b}. However, this process can be time-consuming and costly, especially given the complex nature of modern model architectures that consume significant computational resources~\citep{8787029}. Managing and updating multiple models is therefore a challenge for companies and long-term sustainability is one of the main difficulties faced by industrial machine learning practitioners today~\citep{ref:Paleyes:2022}.

To address this, differential replication through copying~\citep{ref:Unceta:2020} has been proposed as a more efficient and effective solution to adapt models. This approach builds upon previous ideas on knowledge distillation~\citep{ref:Hinton:2015, ref:Bucilua:2006}, and allows for reusing a model's knowledge to train a new one that is better suited to the changing environment~\citep{ref:Unceta:2020b}. It can therefore bring numerous benefits in terms of cost and optimization of resources. 

Differential replication allows for model adaptation by projecting an existing decision function onto a new hypothesis space that meets new requirements. Typically, this process involves using the label probabilities produced by the given decision function as soft targets to train a new model in the new hypothesis space~\citep{ref:Liu:2018, ref:Wang:2020}. In the case of differential replication through copying, information about the original model's behavior is acquired via a hard-membership query interface and training is done using synthetic samples labeled by the original model.

Theoretically, the copying problem can be viewed as a dual optimization problem where both the copy model parameters and the synthetic samples are optimized simultaneously. Previous practical implementations of this problem have simplified it, generating and labeling a large set of synthetic data points in a single pass and then using them to optimize the parameters. This approach has been successful in validating copying on several datasets, but it is memory-intensive and computationally expensive and requires pre-setting several hyperparameters.

This article presents a novel approach to differential replication through copying that is based on an iterative scheme. The goal of performing a copy is to find the simplest model in the model copy hypothesis space that attains the maximum fidelity compared to the target model being copied. In the absence of data, this requires finding the best synthetic set for optimizing the model and the simplest model that guarantees perfect fidelity within the capacity of the copy model space. The proposed iterative formulation performs two steps at each iteration: (1) generating and selecting data for copying based on a compression measure of uncertainty and (2) learning the target copy model using the optimized dataset while controlling its complexity to achieve close-to-optimal results. This process allows for control over the amount of data/memory needed and the convergence speed to reach a steady performance. Results show that the proposed model requires $85-93\%$ less data samples/memory and has an average $80\%$ improvement in convergence speed compared to the single-pass approach, with no significant degradation in performance.

The main contributions of this article are: (1) to the best of our knowledge, this is the first algorithm to specifically address the problem of copying as described in ~\cite{ref:Unceta:2020}; (2) the proposed formulation and algorithms allow for explicit control of memory requirements and convergence speed while maintaining accuracy; and (3) the resulting algorithm is accurate, fast, and memory-efficient, as validated by successful results. The algorithm proposed has two hyperparameters, but this article also presents an algorithm that automatically sets one of them and dynamically adapts it during the learning process, ensuring fast convergence to an accurate solution. The resulting algorithm overcomes some of the limitations of current copying methods and opens up opportunities for its use in various real-life applications.

The rest of this paper is organized as follows. Section 2 addresses the issue of differential replication through copying, providing an overview of relevant methods, and introduces the single-pass approach as the simplest solution. Section 3 presents the sequential approach to copying. It starts by demonstrating the convergence of this approach to the optimal solution. It then introduce a sample removal policy based on an uncertainty measure. It ends by proposing a regularization term to prevent model forgetting. Section 4 empirically demonstrates the validity of the sequential approach through experiments on a large set of 58 UCI problems. The performance of the sequential copying framework is measured in terms of accuracy, convergence speed, and efficiency. The section ends with a discussion of the results. Finally, Section 5 summarizes the findings and outlines future research directions.

\section{Background on copying}
\label{sec:background}
The problem of \textit{environmental adaptation} introduced by~\citep{ref:Unceta:2020b} refers to situations where a machine learning model that was designed under a set of constraints must fulfill new requirements imposed by changes in its environment. The model needs to adapt from a \textit{source scenario}, $s$ where it was trained, to a \textit{target scenario}, $t$, where it is being deployed. 

\subsection{The problem of environmental adaptation}

Formally speaking, environmental adaptation is defined as follows: consider a task $\mathcal{T}$ and a domain $\mathcal{D}$. A trained model $h \in \mathcal{H}_s$ is designed to solve $\mathcal{T}$ in $\mathcal{D}$ under a  set of constraints $\mathcal{C}_s$ and a compatible hypothesis space $\mathcal{H}_s$. The problem of environmental adaptation arises when the original set of constraints $\mathcal{C}_s$ evolves to a new set $\mathcal{C}_t$. Under these circumstances, a potentially different hypothesis space $\mathcal{H}_t$ has to be defined in the same domain $\mathcal{D}$ and for the same task $\mathcal{T}$, which is compatible with the new constraints set $\mathcal{C}_t$~\footnote{In some cases the source and target hypothesis spaces $\mathcal{H}_s$ and $\mathcal{H}_t$ may be the same. However, in the most general case, where the new set of constraints defines a new set of feasible solutions, they are not.}. Unless the considered model $h$ is compatible with the new set of constraints, it will be rendered as obsolete, i.e., $h$ will no longer be a feasible solution. Hence, environmental adaptation refers to the need to adapt $h$ to the constraints introduced by $\mathcal{C}_t$, as shown in the equations of Table~\ref{tab:environmental}.

\begin{table}
\centering
\begin{tabular}{ccc}
\begin{tabular}{ll}
{Source Scenario}\\
$\text{for} \ \mathcal{T} \ \text{in}\; \mathcal{D}$\\
\\
$\underset{\tiny \text{for} \; h\in \mathcal{H}_s}{\text{maximize}}$ & $\mathsf{P}(y|x;h)$\\
subject to &$\mathcal{C}_s$ \\

\end{tabular}
& 
$\rightarrow$
&
\begin{tabular}{ll}
{Target Scenario}\\
$\text{for} \ \mathcal{T} \ \text{in}\; \mathcal{D}$\\
\\
$\underset{\tiny for \; h\in \mathcal{H}_t}{\text{maximize}}$ & $\mathsf{P}(y|x;h)$\\
subject to &$\mathcal{C}_t$\\
\end{tabular}
\end{tabular}
\caption{The problem of environmental adaptation looks for the solution in the same domain and task when the feasibility constraints change.}\label{tab:environmental}
\end{table}

%A possible solution to the problem of environmental adaptation is projecting the original model into a new feasible set compatible for the considered task and domain, and belonging to the target hypothesis space. If we only observe the outcomes, the optimal solution displays the same or the closest decision boundary under the new constraints and hypothesis set.

This problem is different from \textit{domain adaptation} and \textit{transfer learning}~\citep{chen2021,chen2022,raffel2020}. Domain adaptation refers to adapting a model from one source domain $\mathcal{D}_s$ to a related target domain $\mathcal{D}_t$, due to a change in the data distributions. Environmental adaptation preserves the domain, but there is a change in constraints. Transfer learning requires reusing knowledge from solving one task $\mathcal{T}_s$ to solve a related task $\mathcal{T}_t$~\citep{pan2010}. Environmental adaptation preserves the task.

The environmental adaptation problem can be addressed through various methods, including re-training the existing model with a new set of constraints~\citep{barque2018}, using wrappers~\citep{mena2019,mena2020}, edited or augmented data subsets~\citep{song2008,dataaug2020,duan2018}, teacher-student networks and distillation mechanisms~\citep{bucilua2006, hinton2015, szegedy2016, yang2018}, label regularization~\citep{muller2019, yuan2020}, label refinement~\citep{bagherinezhad2018}, or synthetic data generation~\citep{bucilua2006, zeng2000}. A comprehensive overview of all the different methods is available in~\citep{ref:Unceta:2020b}. Here, we focus in differential replication, a general solution to the environmental adaptation problem.

Differential replication projects the decision boundary of an existing model to a new hypothesis space that is compatible with the target scenario. In the absence of access to the training dataset or model internals, differential replication through copying can be used to solve the environmental adaptation problem~\cite{ref:Unceta:2020}. In classification settings, copying involves obtaining a new classifier that displays the same performance and decision behavior as the original, without necessarily belonging to the same model family.

In the following sections, we introduce the problem of differential replication through copying and explore potential approaches for implementation.

\subsection{Differential replication through copying}\label{sec:diff}

Consider a classifier $f_{\mathcal{O}} \in \mathcal{H}s$ trained on an unknown dataset with input space dimensionality $d$ and output space cardinality $n_c$. Thus, $f{\mathcal{O}}:\mathbb{R}^d\rightarrow{0,1}^{n_c}$. In the most restrictive case, $f_{\mathcal{O}}$ is a hard decision classifier that outputs one-hot encoded label predictions, meaning that for any data point, it returns an $n_c$-output vector with all elements as zeros except for the target label position $i$, which has a value of 1.

The goal of copying is to obtain a new classifier $f_{\mathcal{C}} \in \mathcal{H}t$ parameterized by $\theta \in \Theta$ that mimics $f{\mathcal{O}}$ across the sample space. In the empirical risk minimization framework, we can consider an empirical risk function that measures the discrepancy between two classifiers. We can then formulate the copying problem as a dual optimization of $\theta$ and a set of synthetic data points $S$ over which the empirical risk is evaluated, since we do not have access to any training data. This problem can be written as:

%In terms of the empirical risk minimization framework, the problem of copying can be formalized as follows. Consider a model a trained model $f_{\mathcal{O}} \in \mathcal{H}_s$, for $\mathcal{H}_s$ the source hypothesis space, and an empirical risk function that measures the discrepancy between two models. The goal of copying is to obtain a new model $f_{\mathcal{C}} \in \mathcal{H}_t$, parameterized by $\theta$, for $\mathcal{H}_t$ the target hypothesis space that minimizes this function. This amounts to solving a dual optimization problem over both the unknown parameters copy $\theta$ and a set of synthetic data points $\mathcal{Z}$\footnote{Remember that we do not have access to any training data} over which the empirical risk is evaluated. This problem can be written as follows:
\begin{flalign}\label{eq:capacity}
\underset{\theta,S}{\text{minimize}}  &\quad \Omega(\theta)\\
\text{subject to} & \quad \|R^{\mathcal{F}}_{emp}(f_{\mathcal{C}}(\theta),f_{\mathcal{O}})-R^{\mathcal{F}}_{emp}(f_{\mathcal{C}}(\theta^{\dagger}),f_{\mathcal{O}})\|<\varepsilon, \nonumber
\end{flalign}

\noindent
for a defined tolerance $\epsilon$ and a complexity measure $\Omega(\theta)$ (e.g. the $\ell_p$-norm of the parameters). The empirical risk function $R^{\mathcal{F}}{emp}$ measures the difference between the original model $f{\mathcal{O}}$ and the optimized copy model $f_{\mathcal{C}}$ and is referred to as the \textit{empirical fidelity error}. $\theta^{\dagger}$ is the solution to the following unconstrained optimization problem:

\begin{flalign}\label{eq:unconstrained}
\theta^{\dagger} = \underset{\theta,S}{\text{argmin}} & \quad R^{\mathcal{F}}_{emp}(f_{\mathcal{C}}(\theta),f_{\mathcal{O}}).
\end{flalign}

The solution to Equation~\ref{eq:capacity} is a combination of synthetic data and copy model parameters that minimize capacity and empirical risk. If $f_{\mathcal{O}} \in \mathcal{H}_t$, the solution of the unconstrained problem (Equation \ref{eq:unconstrained}) will always result in $R^{\mathcal{F}}{emp}(f_{\mathcal{C}}(\theta^{\dagger}),f_{\mathcal{O}}) = 0$, as the labeling problem for any set of data points using the original hard-label classifier is separable in $\mathcal{H}_t$.

\subsection{The single-pass approach}

The \textit{single-pass} approach~\citep{ref:Unceta:2020b} aims to find a sub-optimal solution to the copying problem by dividing it into two separate steps. Firstly, a synthetic set $S^*$ is found and then, the copy parameters $\theta^*$ are optimized using this set. The process is as follows:

\begin{itemize}
\item {\bf Step 1: Synthetic sample generation}. An exhaustive synthetic set $S^*$ is created by randomly sampling from a probability density function $P_S$ that covers the operational space of the copy. The operational space is the region of the input space where the copy is expected to mimic the behavior of the original model. The synthetic set can be expressed as:
    
    \begin{flalign}
         S^* =\{z_j | z_j\sim P_\mathcal{S},\;j = 1,\dots,N\}. 
    \end{flalign}

    The simplest choice for $P_S$ is a uniform distribution, or a normal or Gaussian distribution if the original data is normalized. (See~\citep{diego} for more information on different options for $P_S$.)

\item {\bf Step 2: Building the copy}. The optimal parameter set for the copy is obtained by minimizing the empirical risk of the copied model over the synthetic set $S^*$ obtained in Step 1:

    \begin{flalign}
    \theta^* = \underset{\theta}{\text{argmin}} & \quad R^{\mathcal{F}}_{emp}(f_{\mathcal{C}}(\theta),f_{\mathcal{O}})\Bigg |_{S=S^*}.
    \end{flalign}
\end{itemize}

The single-pass approach is a simplified solution for the problem modeled in Equation~\ref{eq:capacity}. Step 1 generates data without any optimization, but it requires a large dataset. Step 2 focuses on solving the unconstrained version of the copying problem defined in Equation~\ref{eq:unconstrained}. This approach can be used when the classifier complexity can be directly modeled. However, in the general case, it requires setting many critical parameters and selecting an appropriate model to ensure that the unconstrained problem is a good approximation of the general setting described in Equation~\ref{eq:capacity}. To guarantee good performance, a sufficiently large synthetic dataset must be generated.

The implementation of the single-pass approach has limitations. Firstly, the learning process using a one-shot approach with a single model may be limited by the available memory and unable to handle the full dataset. Secondly, keeping a large number of data samples in memory during training is resource-intensive and doesn't guarantee performance. In addition, blindly learning the entire synthetic dataset using a single model can result in inefficiencies. On the other hand, using an online learning strategy can reduce memory usage, but leads to slow convergence to the optimal solution, making the process time-consuming.

To overcome these limitations, a new approach, using an alternating optimization scheme, is introduced. This approach provides a fast and memory-efficient solution to the copying problem, solving Equations \ref{eq:capacity} and~\ref{eq:unconstrained}.

\section{The sequential approach}
\label{sec:sequential}
We introduce two theorems in this section to show that the sequential approach converges to the single-pass approach when conditions are optimal in terms of both parameters and behavior. Next, we provide a practical implementation of the sequential approach and various optimizations to achieve low memory usage and fast convergence. Specifically, we propose that a perfect copy should be able to compress the synthetic data points in its parameters and suggest epistemic uncertainty as a reliable metric for data compression. Based on this, we develop a data selection strategy that filters samples based on their level of compression by the copy model. This leads to a reduced number of data points needed for each learning step. Finally, we introduce an automatic hyper-parameter tuning policy to ensure optimal implementation of the sequential approach in practice. We refer to the resulting algorithm as the \textit{sequential approach} to copying.

\subsection{An alternating optimization algorithm for copying}

We start by introducing a preliminary alternating optimization algorithm for solving Equation~\ref{eq:capacity}. This algorithm alternates between two optimization steps at each iteration $t$:

\begin{itemize}
    \item {\bf Step 1: Sample Optimization}. The optimal synthetic set at iteration $t$ is selected by maximizing the empirical fidelity error for the previous model solution $\theta^*_{t-1}$. That is,
    \begin{flalign}\label{eq:step1}
    S^*_t = \arg\max_{S}\quad R_{emp}^{\mathcal{F}, S}(f_{\mathcal{C}}(\theta),f_{\mathcal{O}})\Bigg |_{\theta = \theta^*_{t-1}}
    \end{flalign}
    \item {\bf Step 2: Copy Parameter Optimization}. The optimal copy parameters $\theta^*_t$ at iteration $t$ over samples $S_t^*$ are obtained by:
    \begin{flalign}\label{eq:step2}
    \underset{\theta}{\text{minimize}}  &\quad \Omega(\theta)\\
    \text{subject to} & \quad \|R^{\mathcal{F}}_{emp}(f_{\mathcal{C}}(\theta),f_{\mathcal{O}})-R^{\mathcal{F}}_{emp}(f_{\mathcal{C}}(\theta^{\dagger}),f_{\mathcal{O}})\|\Bigg |_{S = S^*_t}<\varepsilon. \nonumber
    \end{flalign}
\end{itemize}

The algorithm starts each iteration $t$ by selecting a set of synthetic data points that maximize the empirical risk. These are the points that the copy model from the previous iteration $t-1$ did not model correctly. By reducing the empirical risk to zero, we can minimize the loss function. In Step 2, we minimize the empirical loss over $S_t$ while keeping the copy model complexity as low as possible. The rest of this section focuses on solving Equations \ref{eq:step1} and \ref{eq:step2} under various assumptions.

\subsection{Step 1: Sample optimization}\label{lab:theorems}

We start by introducing a formal sequential sample generation scheme and examining its convergence properties. This will serve as the basis for constructing the final algorithm that solves Equation \ref{eq:step1}. To achieve this, we recast Equation \ref{eq:unconstrained} in a probabilistic context and prove two theorems showing that the sequential copying process converges in both parameters and behavior to the single-pass approach under optimal conditions.

\subsubsection{The sequential framework}

Consider a sequence of finite sets $S_t$ such that 

\begin{equation}
    S_t \subseteq S_{t+1} \subseteq \cdots \subseteq S
\end{equation}\label{eq:subsets}

\noindent
for $t \in \mathbb{N}$, where $|S| = \aleph^0$, which approaches the set $S$ as $t$ increases towards infinity. The sequential framework is based on the notion that, if $S_t$ converges to $S$, then the optimal copy parameters $\theta_t^*$ obtained by optimizing over $S_t$ will converge to $\theta^*$, the optimal parameters over $S$. This approximation can be iteratively obtained by drawing samples from a given probability density function.

To prove this, we cast the unconstrained copying problem in probabilistic terms \footnote{Observe that we can easily recover the empirical risk minimization framework considering probability density functions of the exponential family. Consider the empirical loss defined as $\frac{1}{n}\sum_{i=1}^n \ell(a,b;\theta)$, then 
$$\arg\min_{\theta} \frac{1}{n}\sum_{i=1}^n \ell(a,b;\theta) = \arg\max_{\theta} \frac{1}{n}\sum_{i=1}^n e^{-\gamma\cdot\ell(a,b;\theta)}.$$
} and express it as the solution to the following empirical distributional problem:

\begin{equation}
    \theta^* = \argmax_\theta\sum_{z\in S} \mathcal{P}(\theta|f_{\mathcal{O}}(z),f_{\mathcal{C}}(z)) = \argmax_\theta F(\theta),
    \label{eq:opt}
\end{equation}
where $S ={z | z\sim P_\mathcal{S}}$ is a synthetic dataset of size $|S| = \aleph^0$.

Next, we present Theorem~\ref{theorem:function_convergence}, which shows that the solution to Equation~\ref{eq:opt} for $S$ can be approximated using the sequence of iterative values $S_t$.

\begin{theorem}
Let $S_t\subseteq S_{t+1}\subseteq\cdots\subseteq S$ be a sequence of subsets converging to $S$. Then, the sequence of functions $\big\{F_t\big\}_t$ defined as $F_t(\theta)=\sum_{z\in S_t} \mathcal{P}(\theta|\fc(z,\theta),\fo(z))$, converges uniformly to $F(\theta)=\sum_{z\in S} \mathcal{P}(\theta|\fc(z,\theta),\fo(z))$.
\label{theorem:function_convergence}
\end{theorem}

Theorem~\ref{theorem:params_convergence} proves the convergence of parameters in Theorem~\ref{theorem:function_convergence}.

\begin{theorem}
Under the conditions of Theorem~\ref{theorem:function_convergence}, the sequence of parameters $\big\{\theta_t^*\big\}_t$ defined as $\theta_t^*=\argmax_{\theta \in \Theta} F_i(\theta)$, converges to $\theta^*=\argmax_{\theta \in \Theta} F(\theta)$, where $\Theta$ is the complete set of parameters.
\label{theorem:params_convergence}
\end{theorem}

The full mathematical proof of convergence of the copy parameter sequence can be found in Appendix~\ref{Sec:AppA}.

{\bf Definition:} The {\it sequential approach} learning algorithm is an iterative process that optimizes the copy parameters $\theta_t$ incrementally over sets of synthetic data points $S_t,\; t = 1\dots T$. 

The preliminary version of the sequential approach is outlined in Algorithm~\ref{alg:sequential}. We begin by generating an initial synthetic set $S_0$ of size $n$ and optimizing the copy parameters accordingly. As stated in line 4, at each iteration $t$ the new synthetic set $S_t$ adds $n$ samples to the previous set $S_{t-1}$. The cardinality of $S_t$ is therefore given by $\Big|S_t\Big|=t\cdot n$ and grows linearly with $t$. This linear growth policy was chosen as the simplest strategy to build the subsets in Equation~\ref{eq:subsets}. Other strategies can also be used. Line 5 describes the optimization of the algorithm considering the solution of the previous step. 

\begin{algorithm}[!t]
\caption{Preliminary version of the sequential approach}
\label{alg:sequential}
\begin{algorithmic}[1]
\State \textbf{Input}: \textbf{int} $T$, \textbf{int} $n$, \textbf{Classifier} $f_\mathcal{O}$
\State\textbf{Output}: Optimal copy parameters
\State $S_0 \gets \{(z_j, f_\mathcal{O}(z_j)) \;|\;  z_j \sim P_\mathcal{S}, j=1\dots n \}$
\State $\theta^*_0 \gets \argmax_\theta\sum\limits_{z\in S_0} \mathcal{P}(\theta|f_\mathcal{O}(z),f_\mathcal{C}(z))$
\For{$t=1$ to $T$}
    \State $S_t \gets S_{t-1}\bigcup\;\{(z_j, f_\mathcal{O}(z_j)) \;|\;  z_j \sim P_\mathcal{S}, j=1\dots n \}$
    \State $\theta^*_t \gets \argmax_\theta\sum\limits_{z\in S_t} \mathcal{P}(\theta|f_\mathcal{O}(z),f_\mathcal{C}(z),\theta^*_{t-1})$
\EndFor

\State \Return $\theta^*_t$

\end{algorithmic}
\end{algorithm}

For demonstration, we test the sequential approach on a toy binary classification dataset, the \textit{spirals} problem. We first train a Gaussian kernel SVM on this dataset, achieving perfect accuracy. Then, we copy the learned boundary using a fully-connected neural network with three hidden ReLu layers of 64, 32, and 10 neurons, and a SoftMax output. We train it for 1000 epochs with a batch size of 32 samples. Figure~\ref{fig:spirals}a) shows the original decision boundary and data points.

\begin{figure}[h]
  \centering
  \subfloat[]{
  \includegraphics[width=0.43\columnwidth]{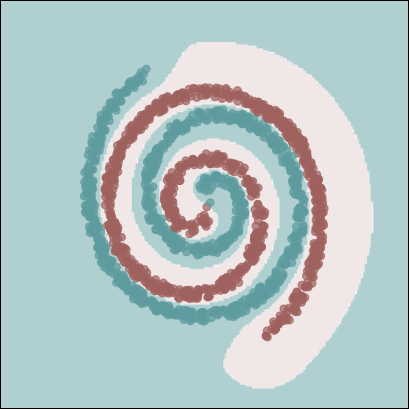}}
  \subfloat[]{\includegraphics[width=0.57\columnwidth]{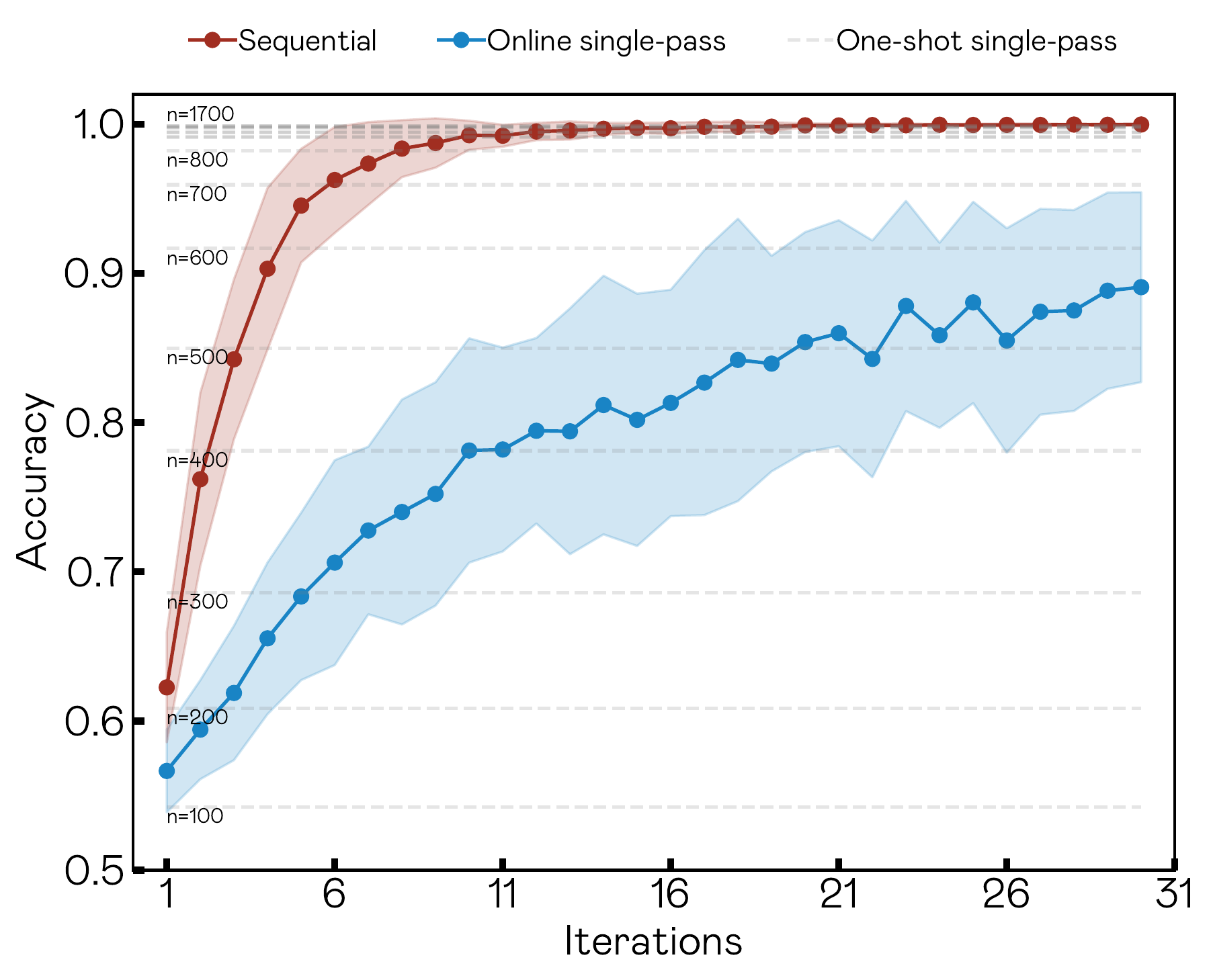}}

    \caption{a) Decision boundary learned by a Gaussian kernel SVM trained on the \textit{spirals} dataset. b) Copy accuracy averaged over 30 independent runs for the sequential approach (orange) and the online single-pass approach (blue) with $n=100$. The grey dashed lines depict the accuracy obtained using a one-shot single-pass approach for various values of $n$. The shaded areas indicate a standard deviation from the mean.}
\label{fig:spirals}
\end{figure}

We compare the results of Algorithm~\ref{alg:sequential} with those of the single-pass approach in both its unique-step and online implementations. Figure~\ref{fig:spirals}b) shows the average accuracy computed at different iterations for copies trained using different methods. The online and sequential approaches are trained by generating $n=100$ new samples at each iteration. The grey dashed lines represent the accuracy achieved by single-pass copies trained in a one-shot for different values of $n$. The reported values represent accuracy measured over the original test data points.

%Statistically, both the single-pass and the sequential approaches perform similarly after being exposed to the same number of data points. At iteration $t=6$ the average accuracy for the sequential approach $0.959$, whereas the single-pass approach with $n=600$ exhibits an average accuracy of $0.916$, indicating that for small values of $n$, the sequential training performs slightly better than the single-pass. Then, for a larger value, $n = 1300$, both show an average accuracy of $0.998$. On the other hand, the online training shows poor performance compared with the other training.

The plot highlights differences in convergence speeds among the three approaches. The one-shot single-pass approach with $n=500$ has an accuracy of $0.85$, while the purely online single-pass model reaches an accuracy of around $0.85$ after 30 iterations and 3000 samples. In comparison, the sequential approach has an accuracy of $0.94$ at $t=5$ with 500 samples, indicating faster convergence. However, both the single-pass and sequential approaches reach an accuracy of $\approx 1$ after being exposed to 1000 samples.

\begin{table}[h!]
  \centering  
            \setlength{\tabcolsep}{5pt}
            \renewcommand{\arraystretch}{1.5}
            \begin{tabular}{c|c|c|c}
                             & One-shot single-pass & Online single-pass & Sequential\\
                             \hline
            Memory & \begin{tabular}[c]{@{}l@{}}Fixed large value\\ Needs to be estimated\end{tabular}  & Fixed small value  & Increase monotonically    \\
            \hline
            Accuracy & High but $N$-dependent &\begin{tabular}[c]{@{}l@{}}Increase with iteration\\ Upper bound is $N$-dependent \end{tabular} & Increase with iteration  \\
            \hline 
            Time & $T\propto O(t\cdot N)$   & $T\propto O(t\cdot N)$  & $T\propto O(t^2\cdot N)$ \\
            \hline
            \end{tabular}
    \caption{Comparison table that summarizes the principal differences between both three learning approaches: one-shot \textit{single-pass}, online \textit{single-pass} and \textit{sequential}.}
\label{tab:tab1}
\end{table}

Table~\ref{tab:tab1} compares the three approaches in terms of memory, accuracy, and computational time. The one-shot single-pass approach requires a large amount of data to achieve high accuracy and estimating the number of training data points can lead to under or over-training. The online single-pass approach alleviates memory allocation issues but has limited accuracy and is dependent on the value of $n$. While it is guaranteed to converge to the optimal solution, this is an asymptotic guarantee. The sequential approach combines the benefits of both: it increases the number of points over time, so there's no need to estimate this parameter upfront, and stops copying once a desired accuracy is reached, as proven by Theorem~\ref{theorem:function_convergence}'s monotonic accuracy increase with iterations. The main drawback of this method is its computational cost. The sequential approach grows quadratically with the number of data points, unlike the linear growth of both the single-pass and pure online approaches. This makes the sequential approach highly time-consuming in its current implementation. This comes as no surprise given the construction of the sequence of the set $\{S_t\}$. We discuss potential improvements to address this issue in what follows.

\subsubsection{A model compression measure}\label{lab:duality}

The purpose of a parametric machine learning model is to compress
the data patterns relevant to a specific task. In the copying framework, the original model is considered the ground truth, and the goal is to imitate its decision behavior. Because this model produces hard classification labels, its decision boundary creates a partition of the feature space that results in a separable problem. Hence, any uncertainty measured during the copying process should only come from the model, not the data: a perfect copy should have no aleatoric uncertainty while compressing all the data patterns.

\noindent
{\bf Assumption:} Any uncertainty measured during the copying process is epistemic\footnote{Epistemic uncertainty corresponds to the uncertainty coming from the distribution of the model parameters. It can be reduced as the number of training samples increases. This allows us to single out the optimal model parameters. Thus, any measurement of uncertainty comes from the mismatch between the model parameters and the desired parameters, i.e. a perfect copy will have zero epistemic uncertainty.}. 

It follows from this assumption that when the epistemic uncertainty over a given data point is minimum, then the data point is perfectly compressed by the model. This is, we can assume that a data point is perfectly learned by the copy when its uncertainty is zero. 

Estimating uncertainty in practice requires that we measure the difference between the predictive distribution and the target label distribution. This can be done using various ratios or predictive entropy. We refer the reader to~\citep{devries2018learning,SENGE201416,nguyen2018reliable} for full coverage of this topic. For the sake of simplicity, here we consider the simplest approach to measuring uncertainty. We relax the hard-output constraint for the copy. Instead, we impose that, for any given data point, the copy returns a $n_c$-output vector of class probability predictions, such that

\begin{equation*}
    f_\mathcal{C}:\mathbb{R}^d\times\Theta\longrightarrow\big[0,1\big]^{n_c}.
    \label{def:copy_model}
\end{equation*}

For the sake of simplicity, we model the \textit{uncertainty}, $\rho$, of the copy for a given data point $z$ as the normalized euclidean norm of the distance between the $n_c$-vectors output by the original model and the copy, such that

\begin{equation}
    \rho(z,\theta) = \frac{\|f_\mathcal{C}(z,\theta) - f_\mathcal{O}(z)\|_2}{\sqrt{n_c}} \in [0,1].
    \label{eq:rho_nseq}
\end{equation}

A small $\rho$ value indicates that the copy has low uncertainty (or strong confidence) in the class prediction output for $z$ and that $z$ is properly classified. Conversely, a large $\rho$ value indicates that, despite the copy model having strong confidence in the class predicted for $z$, this prediction is incorrect. Or alternatively, despite the prediction being correct, there is a large dispersion among the output class probabilities. 

This uncertainty measure can be used to evaluate how well individual data points are compressed by the copy model in the sequential framework. In the limit, this measure can also be used to assess how well the copy model replicates the original decision behavior. Hence, we can introduce a new loss function such that for a given set of data points $S_t$ we define the empirical risk of the copy as

\begin{equation}\label{eq:emprisk}
    R^{\mathcal{F}}_{emp}(f_{\mathcal{C}}(\theta_j),f_{\mathcal{O}}) = \frac{1}{\big|S_t\big|}\sum_{z\in S_t} \rho^2(z,\theta).
\end{equation}

The results of using the loss function are displayed in Figure \ref{fig:sequential_rho}. The plot illustrates the average value of $\rho_t$ per iteration over all data points and runs for the different copying approaches introduced in Figure \ref{fig:spirals}b). The value of $\rho_t$ is significantly low for the sequential learning approach at iteration $t=15$. Conversely, the single-pass and online learning methods exhibit higher uncertainty levels throughout the process. This difference can be attributed to the fact that the sequential approach continuously exposes the model to the same data points, reducing uncertainty iteration by iteration. The other two methods, however, only expose the model once to the same data, causing higher uncertainty levels. The one-shot single-pass approach exposes the model to the whole dataset at once, while the sequential approach never exposes the model to the same data points more than once. As a result, the sequential approach has higher redundancy and reduces uncertainty faster when compared also to the online method, where the learner is never exposed to the same data points more than once.

\begin{figure}[h!]
  \centering
  \includegraphics[width=0.75\columnwidth]{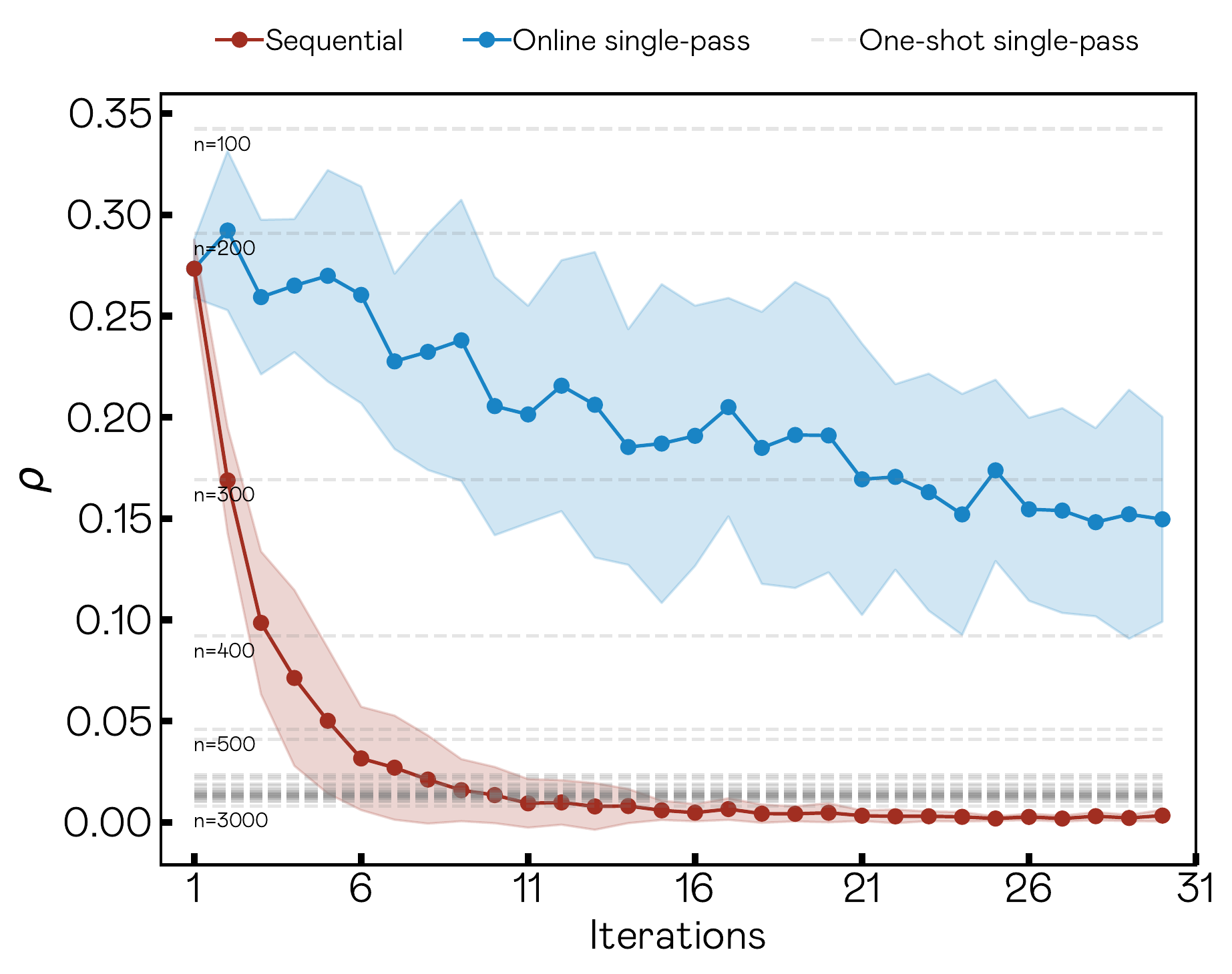}
    \caption{Average uncertainty across the entire dataset for both the sequential (orange) and online single-pass (blue) copying approaches, averaged over 30 independent runs. The uncertainty achieved by the one-shot single-pass approach, with different values of $n$, is represented by dashed lines. The shaded area indicates a one standard deviation deviation from the mean.}
\label{fig:sequential_rho}
\end{figure}

%Figure \ref{fig:fig2_rho}b) shows the evolution of the uncertainty computed over the first dataset, $S_1$, at each iteration. The uncertainty before the training is large for both sequential and online learning methods and drastically drops after the models are trained. However, whereas the sequential learning keeps the points, the uncertainty over $S_1$ remain as low as after the first iteration, the online learning never expose the model to $S_1$ again and the uncertainty dramatically grows to larger values at the next iteration. Therefore, seems clear that keeping the points give us the best accuracy-confidence relation at expense of computational time, as we saw in Table \ref{tab:tab1}. 

%The notion of uncertainty is also used in the next subsection as a selection criterion for solving Equation \ref{eq:step1}. 

\subsubsection{A sample selection policy for copying}\label{lab:filtering}

Including an uncertainty measure in the training algorithm enables us to assess the degree of compression for each data point. The higher the compression, the better the copy model is at capturing the pattern encoded by each data point for the given task. This estimation can be used for selecting those points that contribute the most to the learning process. By filtering out the rest of the samples, we can reduce the number of resources consumed when copying. Hence, we enforce a policy that uses uncertainty as a criterion for sample selection. 

At each iteration, data points with an uncertainty lower than a threshold $\delta$ are removed from the learning process (refer to Algorithm~\ref{alg:filtering}). The procedure starts with building a new sequential set by randomly sampling $n$ points and adding them to $S_{t-1}$ in line 1. Then, in line 2, the uncertainty measure is used to select points with $\rho_j \leq \delta$, forming the filtered set $S_t$ that is used to optimize the copy parameters.

Figure~\ref{fig:drooping_no_lambda} presents the results for the sequential training with $n=100$ and different values of the sample selection parameter $\delta$ in the \textit{spirals} dataset. For comparison, results for the online single-pass approach and the pure sequential approach are also shown. Figure \ref{fig:drooping_no_lambda}a) displays the average accuracy at each iteration. The results demonstrate that sequential training with sample selection performs better than online training, but falls short of the pure sequential setting.

Figures~\ref{fig:drooping_no_lambda}b) and c) show the change in $\rho$ and number of synthetic data points used for training over increasing iterations, respectively. The online single-pass approach shows a constant uncertainty, while the pure sequential approach approaches 0. Contrary to what one might expect, the sample selection policy leads to an overall increase in uncertainty, while also reducing the number of data points used for training. Eventually, the number of samples $Nn$ converges to a fixed value after a few iterations.

\begin{algorithm}[ht]
\caption{Sample selection policy}
\label{alg:filtering}
\begin{algorithmic}[1]
\State \textbf{Input}: \textbf{Sample Set} $S_{t-1}$, 
\textbf{int} $n$, \textbf{float} $\delta$, \textbf{Distribution} $P_\mathcal{Z}$, \textbf{Measurement} $\rho()$, \textbf{Classifier} $f_\mathcal{O}$, \textbf{Classifier} $f_\mathcal{C}(\theta_{t-1}^*)$
\State \textbf{Output}: Filtered set
\State $ \Xi \gets S_{t-1}\bigcup\;\{(z_j, f_\mathcal{O}(z_j)) \;|\;  z_j \sim P_\mathcal{Z}, j=1\dots n \}$
    \State $S_t \gets \{z| z\in \Xi, \rho(z,\theta^*_{t-1})\geq \delta\}$ \Comment{Filter according to uncertainty}

\State \Return $S_t$

\end{algorithmic}
\end{algorithm}

\begin{figure}[ht]
  \centering
  \subfloat[]{ \includegraphics[width=0.47\columnwidth]{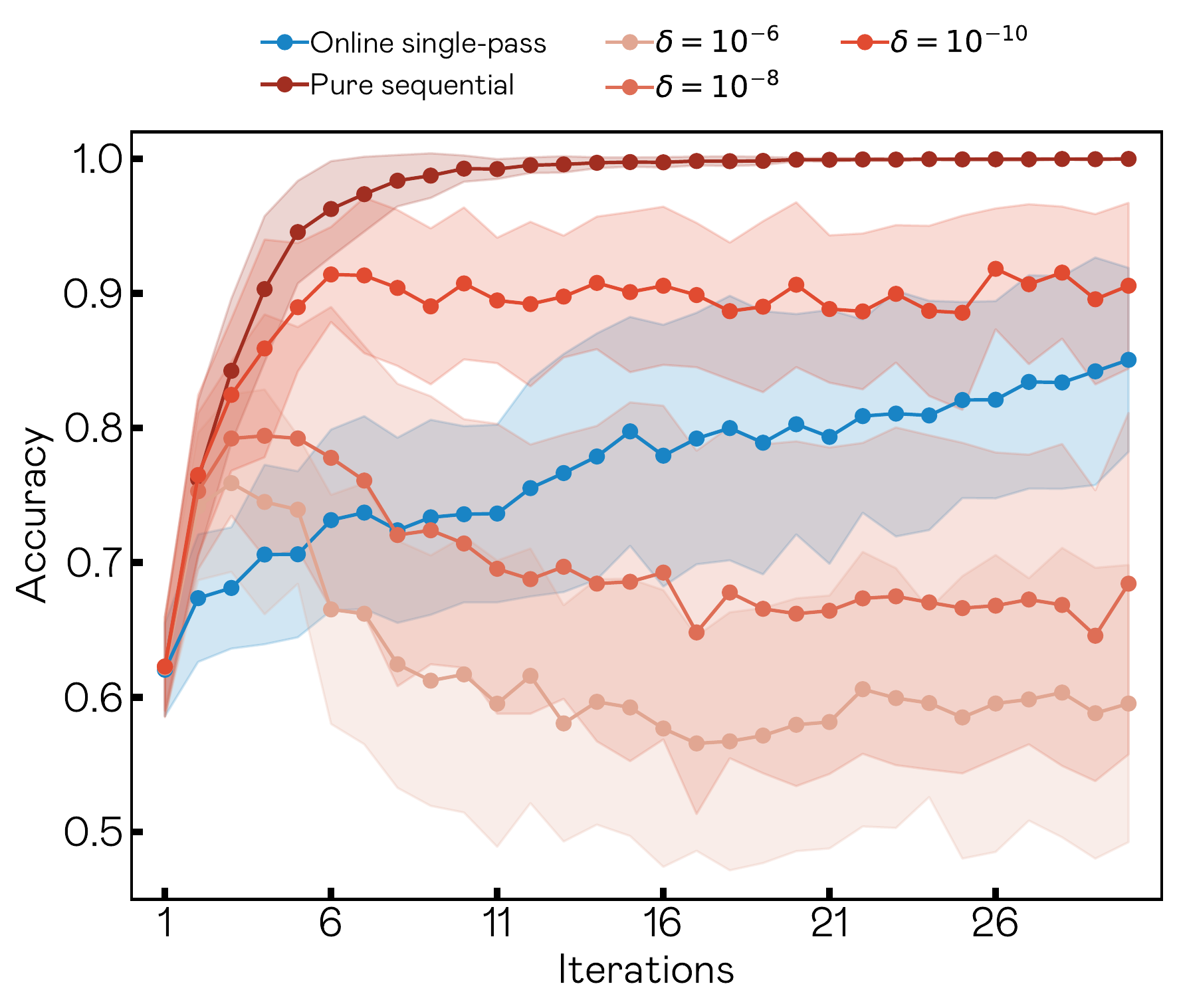}}
  \subfloat[]{ \includegraphics[width=0.47\columnwidth]{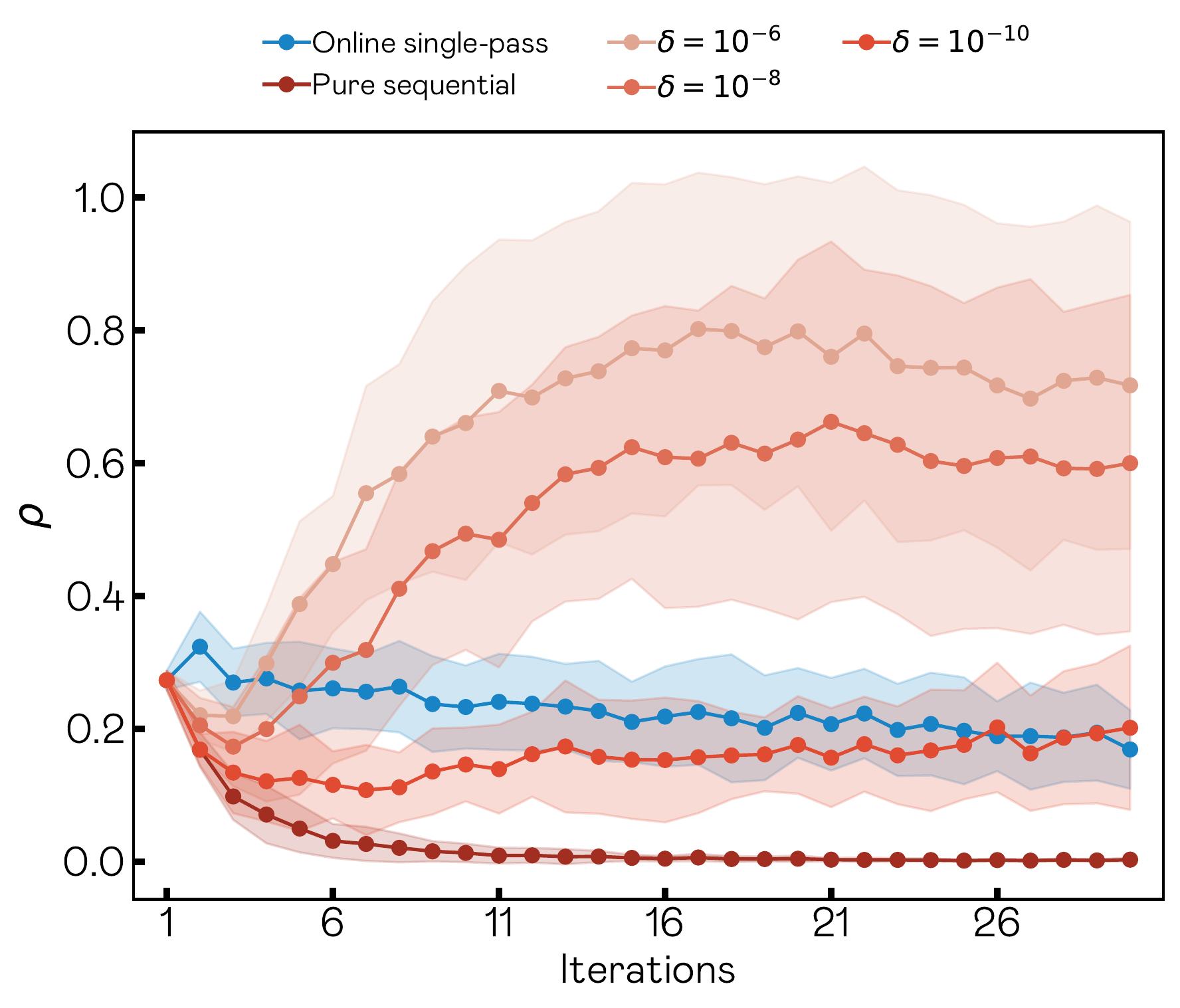}}
  
  \subfloat[]{ \includegraphics[width=0.47\columnwidth]{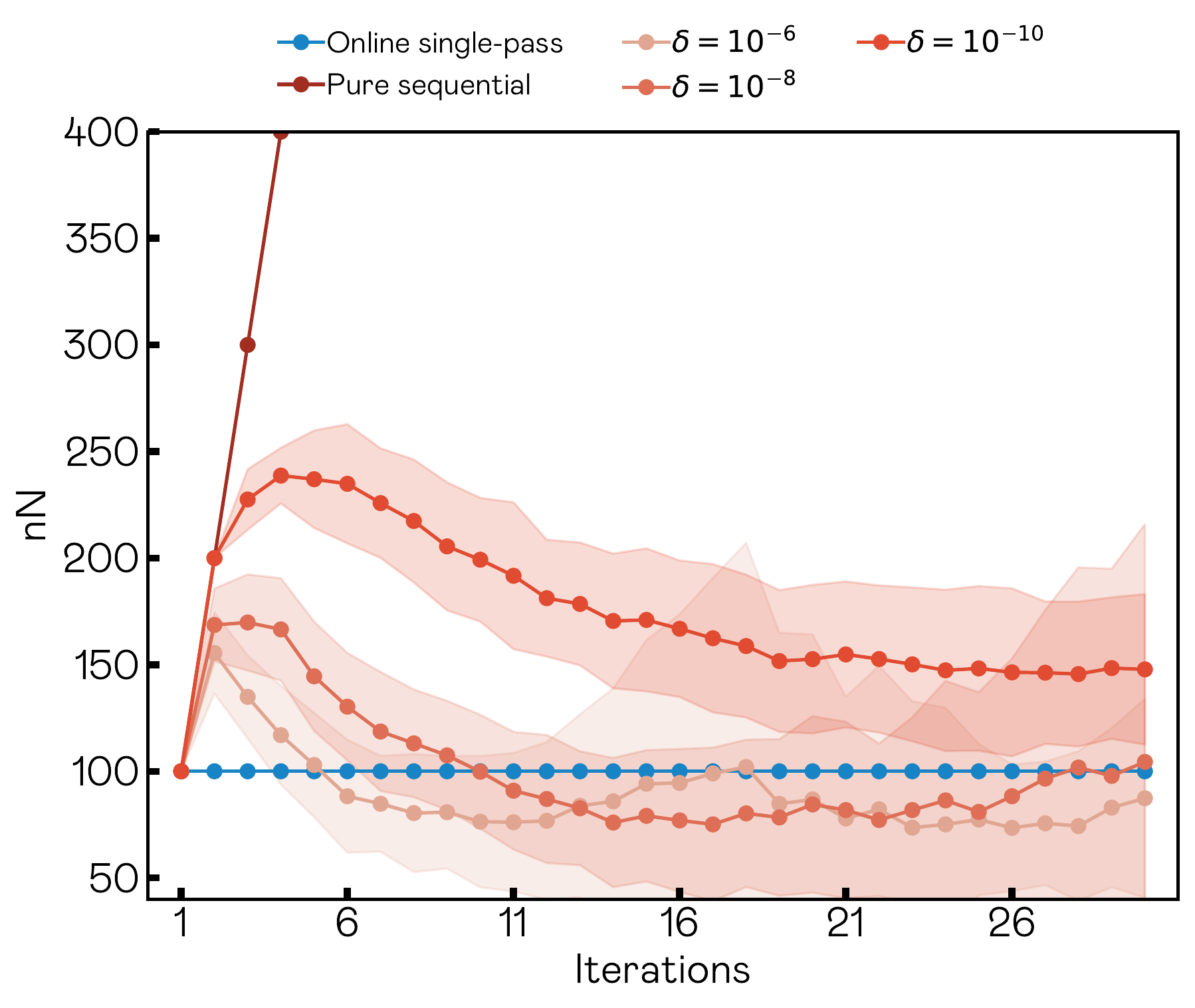}}
    \caption{a) Accuracy and b) uncertainty per iteration for the sequential approach with different uncertainty thresholds, averaged over 30 independent runs. Results for the online setting are also shown in blue for comparative purposes at every iteration. c) Number of data points used at each iteration. The $y$-axis is restricted to $400$ in order to make curves observable. The number of points of the sequential learning (orange line) grows linearly until $30\cdot100 = 3000$.} \label{fig:drooping_no_lambda}
\end{figure}

We consider all three plots in Figure~\ref{fig:drooping_no_lambda} at once. During the first iterations, when copies are still not sufficiently tuned to the data, there is no filtering effect. The number of points $Nn$ grows almost equally for all the settings for a few iterations, as shown in Figure \ref{fig:drooping_no_lambda}c). During this time, accuracy increases gradually, as shown in Figure \ref{fig:drooping_no_lambda}a), and $\rho$ decreases, as displayed in Figure \ref{fig:drooping_no_lambda}b), as expected for a model that is compressing information. The differences arise after a few iterations when the filtering effect begins and the number of data points decreases in settings where the sample selection policy is enforced. At this stage, the copy models have a confidence level close to $\delta$, which allows the removal of samples. As a result, accuracy stops growing and becomes flat, while average uncertainty starts to rise. With regard to the number of points, the curves quickly stabilize by $t=18$. The model with the lowest threshold, $\delta = 10^{-10}$, accumulates the larger number of points ($\approx 250$). This is less than 10$\%$ of the number of points required by the pure sequential method, but still reaches a reasonable accuracy. On the other hand, models trained with $\delta=10^{-8}$ and $\delta=10^{-6}$ have lower accuracy compared to the online single-pass setting.

This algorithm addresses the problem in Equation~\ref{eq:step1} but with limitations. The uncertainty metric $\rho$ is linked to the empirical fidelity error $\mathcal{R}^{FC}_{emp}$. The points with high uncertainty also have the highest empirical fidelity error. However, removing points from set $S_t$ violates the assumptions in Theorems~\ref{theorem:function_convergence} and~\ref{theorem:params_convergence}. When the copy model has high confidence in the synthetic data, sample removal changes its behavior from sequential (adding points incrementally) to online (fixed number of points per iteration). For parametric models, this shift results in \textit{catastrophic forgetting}, as shown in Figure \ref{fig:drooping_no_lambda} where sequential models start to act like online models. The removal of samples increases uncertainty in the copy models, causing them to forget prior information. To address this, we introduce several improvements in the implementation of Step 2 in the alternating optimization algorithm.

\subsection{Step 2: Optimizing copy model parameters}

Let us briefly turn our attention to Equation \ref{eq:step2}. This equation models the challenge of controlling the capacity of the copy model while minimizing the loss function. To tackle this issue, we introduce a capacity-enhancement strategy. For this purpose, it's worth noting that if we assume that the copy model has enough capacity, then Equation~\ref{eq:step2} can be simplified for a given iteration $t$ as follows
    
    \begin{flalign}
    \underset{\theta}{\text{minimize}}  &\quad \Omega(\theta)\\
    \text{subject to} & \quad \|R^{\mathcal{F}}_{emp}(f_{\mathcal{C}}(\theta),f_{\mathcal{O}})\|\Bigg |_{S = S^*_t}<\varepsilon. \nonumber
    \end{flalign}

Having the above simplification in mind, our proposed scheme, outlined in Algorithm \ref{alg:retraining}, aims to control the copy model's capacity while minimizing the loss function. It iteratively solves the parameter optimization problem in stages, ensuring that the empirical risk decreases at each step. At iteration $t$, the copy model starts with a small capacity $\Omega$ and solves the optimization problem for increasing upper bounds $\epsilon_k$ of the target value $\varepsilon$. For increasing values of $k$, we increase the capacity and reduce the upper bound. The larger the capacity, the more the empirical risk decreases and the smaller the complexity of the approximation to the target value.

In practice, we train copies using the loss function in Equation \ref{eq:emprisk}. Using stochastic subgradient optimization models (neural networks with "Stochastic Gradient Descent"), we control model complexity with an early stopping criterion. Thus, delaying early stopping increases model complexity. To improve convergence, hyperparameters are updated at each iteration.

\begin{algorithm}[!t]
\caption{Empirical risk minimization implementation}
\label{alg:retraining}
\begin{algorithmic}[1]
\State \textbf{Input}: \textbf{Sample Set} $S$, \textbf{float} $\varepsilon$,\textbf{Classifier} $f_\mathcal{O}$
\State \textbf{Ouptut}: Copy parameters at step $k$
\State $k \gets 0$, $\epsilon_k \gets C, \; C>>\varepsilon$ \Comment{$C$ takes an arbitrarily large value.}

\While{$R^{\mathcal{F}}_{emp}(f_{\mathcal{C}}(\theta_k),f_{\mathcal{O}};\Omega)\geq \varepsilon$ }
    \State $\theta^*_k \gets \underset{\theta}{\arg\min}\; R^{\mathcal{F}}_{emp}(f_{\mathcal{C}}(\theta),f_{\mathcal{O}};\Omega)$ 
    subject to 
    $R^{\mathcal{F}}_{emp}(f_{\mathcal{C}}(\theta),f_{\mathcal{O}};\Omega)\geq \epsilon_k \Bigg |_{\theta^0_k = \theta^*_{k-1}}$
    \State $k \gets k + 1$
    \State $\epsilon_k \gets \min(\epsilon_{k-1}/2,\varepsilon)$ \Comment{Reduce the value of epsilon.}
    \State Increase $\Omega$
   
\EndWhile

\State \Return $\theta^*_k$

\end{algorithmic}
\end{algorithm}

\subsubsection{Forcing the model to remember}\label{lab:forgetting}

With the previous scheme in mind, we revisit the argument in Section~\ref{lab:filtering}. As noted, the sample removal policy conflicts with some of the assumptions made in the theorems introduced in Section~\ref{lab:theorems}, leading to a forgetting effect in Step 2. To restore the convergence properties, we first point out that, as shown by Theorem~\ref{theorem:params_convergence}, parameters converge to their optimal value in the sequential approach. This implies that the difference between two consecutive terms in the sequence also converges to zero, such that

\begin{equation}
    \big|\big| \theta_{t+1}^*-\theta_t^* \big|\big| \longrightarrow 0.
    \label{eq:params_conv_regularization}
\end{equation}

The asymptotic invariant in Equation~\ref{eq:params_conv_regularization} can be forced to preserve the compressed model obtained from previous iterations even after filtering the data points. We add a regularization term to the loss function at iteration $t$, which minimizes the left term

\begin{equation}
    \mathcal{L}_t = R^{\mathcal{F}}_{emp}(f_{\mathcal{C}}(\theta_t),f_{\mathcal{O}};\Omega)+ \lambda\cdot \| \theta_t-\theta_{t-1}^*\|
\end{equation}

This regularization term originates from the derived theorems and can also be found in the literature under the name of Elastic Weight Consolidation (EWC), though derived heuristically from a different set of assumptions \citep{Kirkpatrick2017}. Algorithm~\ref{alg:retraining_regu} outlines the implementation of this strategy.

\begin{algorithm}[!t]
\caption{Memory aware empirical risk minimization implementation}
\label{alg:retraining_regu}
\begin{algorithmic}[1]
\State \textbf{Input}: \textbf{Sample Set} $S$, \textbf{float} $\varepsilon$,\textbf{Classifier} $f_\mathcal{O}$, \textbf{Parameters} $\theta^*_{t-1}$
\State \textbf{Output}: Copy parameters at step $k$
\State $k \gets 0$, $\epsilon_k \gets C, \; C>>\varepsilon$ \Comment{$C$ takes an arbitrarily large value.}

\While{$R^{\mathcal{F}}_{emp}(f_{\mathcal{C}}(\theta_j),f_{\mathcal{O}};\Omega)\geq \varepsilon$ }
    \State \begin{flalign*}
        \theta^*_k \gets \underset{\theta}{\arg\min}\; R^{\mathcal{F}}_{emp}(f_{\mathcal{C}}(\theta),f_{\mathcal{O}};\Omega)+ \lambda\cdot \| \theta-\theta_{t-1}^* \|  \\
        \text{subject to }R^{\mathcal{F}}_{emp}(f_{\mathcal{C}}(\theta),f_{\mathcal{O}};\Omega)\geq \epsilon_k \Bigg |_{\theta^0_k = \theta^*_{k-1}}
    \end{flalign*}
    \State $k \gets k + 1$
    \State $\epsilon_k \gets \min(\epsilon_{k-1}/2,\varepsilon)$ \Comment{Reduce the value of epsilon}
    \State Increase $\Omega$
\EndWhile

\State \Return $\theta^*_k$

\end{algorithmic}
\end{algorithm}

Continuing with the spirals example, the experimental results for the sequential copying process with four different $\lambda$ values and a threshold $\delta=10^{-8}$ are shown in Figure~\ref{fig:fixed_lambda} (see also Figure~\ref{fig:drooping_no_lambda} for comparison). As before, Figure~\ref{fig:fixed_lambda}a) reports the change in accuracy for increasing iterations, while Figures~\ref{fig:fixed_lambda}b) and c) present the results for the uncertainty $\rho$ and the number of data points used $nN$. Again, we can combine the information from all three panels to better understand the behavior displayed by the copy model when \textit{forced to remember}. Initially, the model has not been trained, so the number of points increases while the uncertainty decreases. Once the removal threshold $\delta$ is reached, the regularization term becomes relevant. For low $\lambda$ values, the $\rho$-term of Equation~\ref{eq:params_conv_regularization} dominates the cost function, and the optimization process learns the new data points. As a result, the copy parameters adapt to new data points and the model exhibits forgetting. Conversely, for \textit{large} $\lambda$ values, the regularization term will dominate the optimization. The model is therefore forced to retain more data points to ensure the $\rho$ term can compete with the regularization term. In this setting, the number of data points still grows, but in a sub-linear way. The most desirable behavior is a constant number of points during the training process. For this particular example, we observe this behavior for $\lambda = 0.05$ and $\lambda = 0.1$.

\begin{figure}[ht]
  \centering
  \subfloat[]{ \includegraphics[width=0.47\columnwidth]{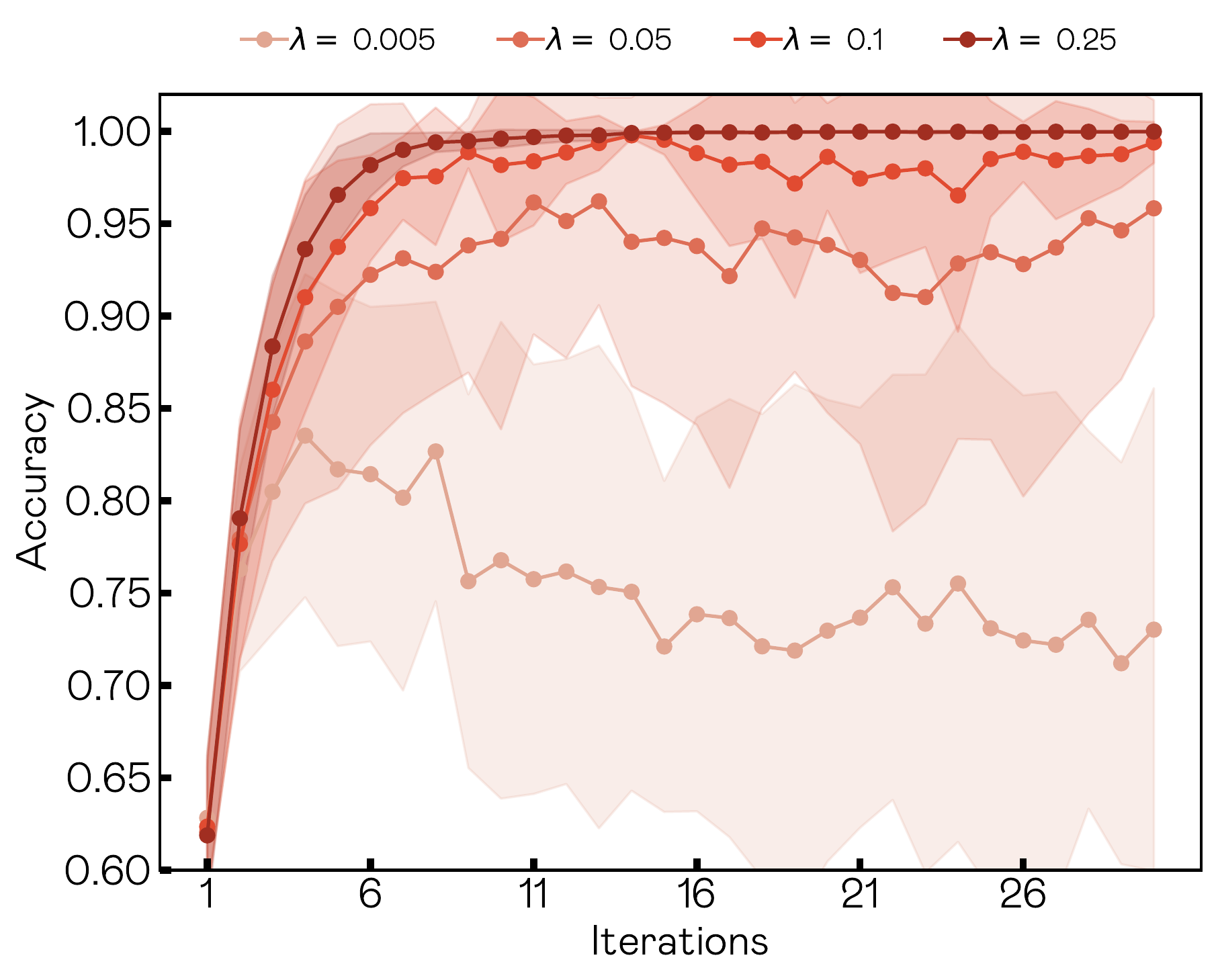}}
  \subfloat[]{ \includegraphics[width=0.47\columnwidth]{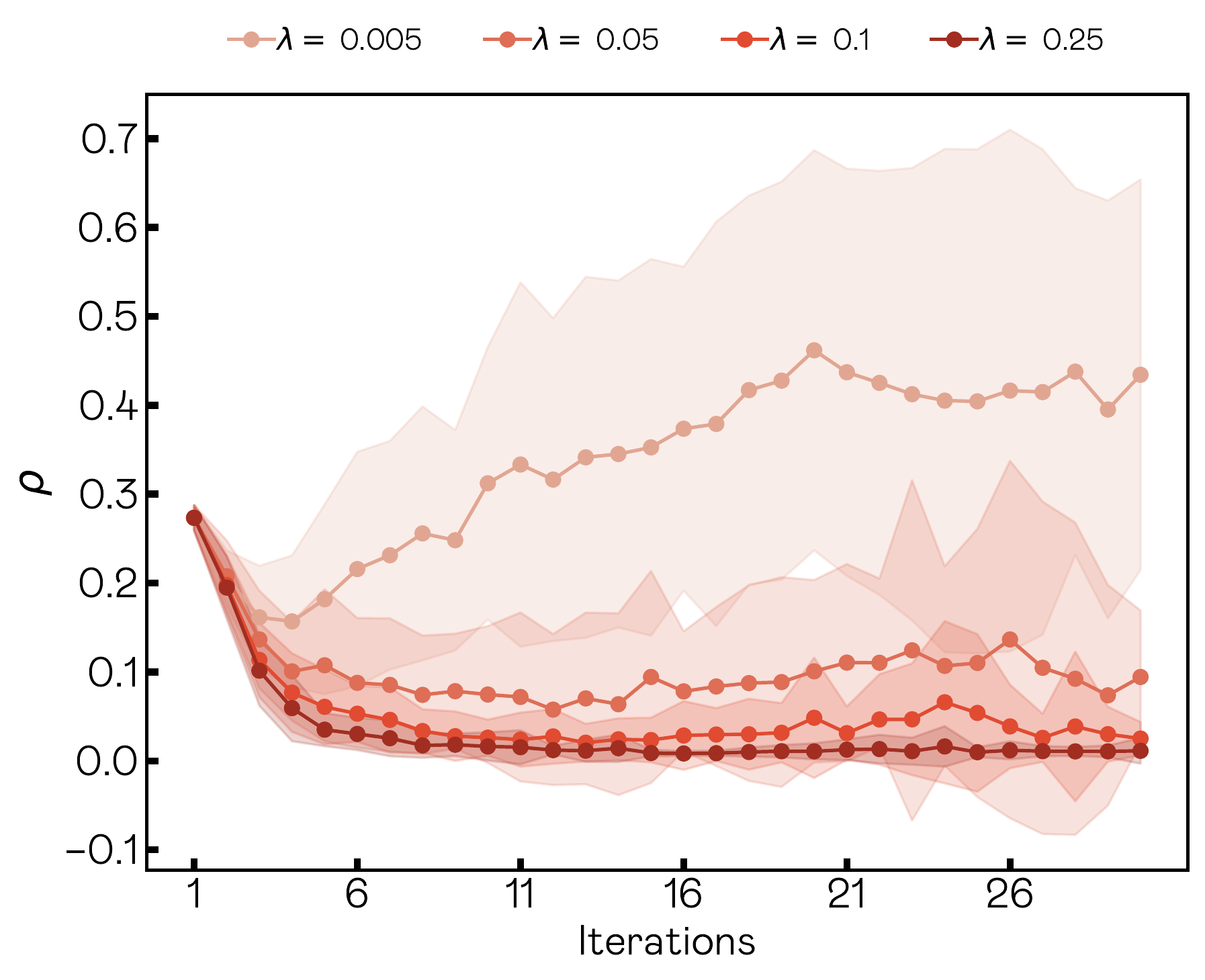}}
  
  \subfloat[]{ \includegraphics[width=0.47\columnwidth]{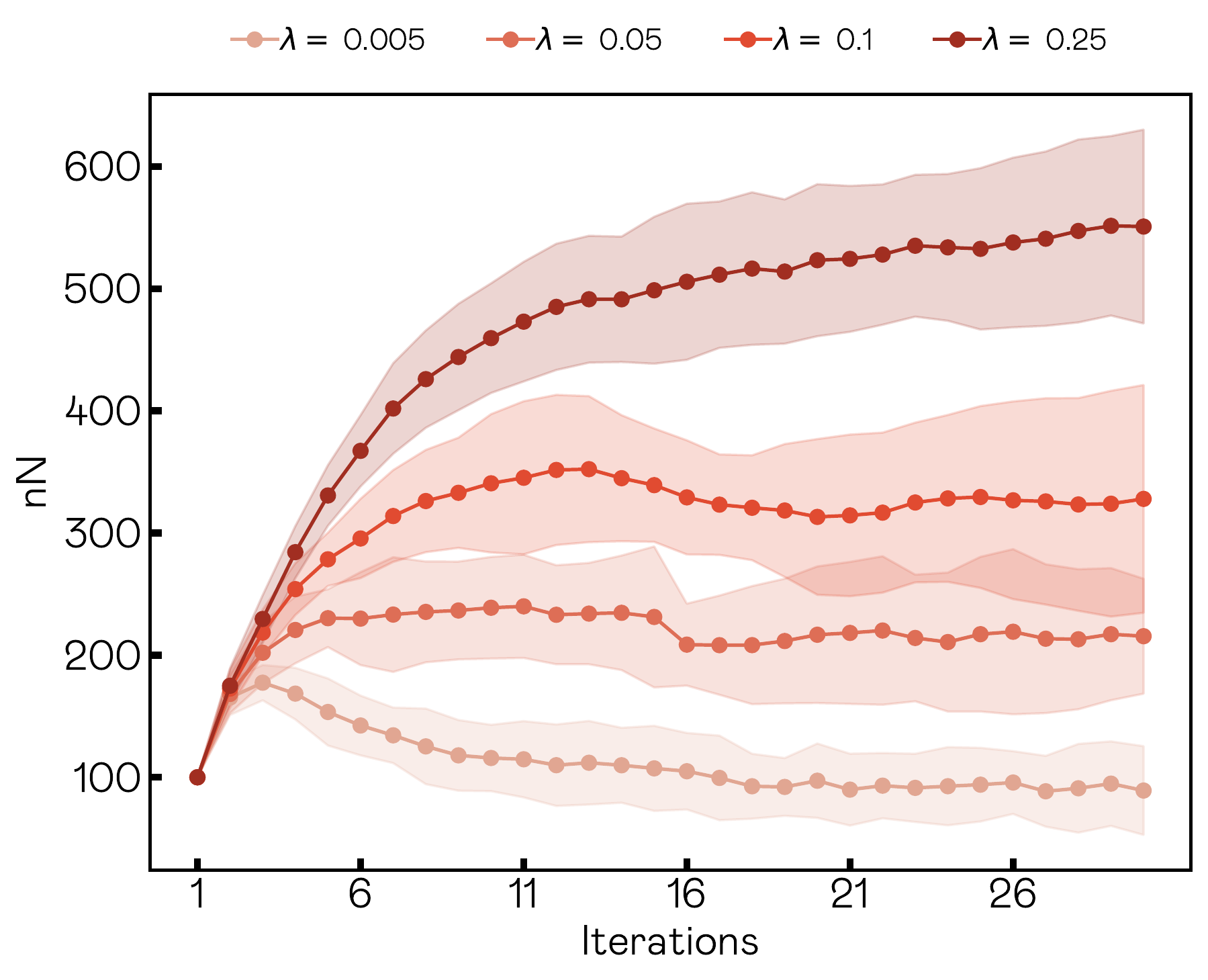}}
    \caption{a) Copy accuracy averaged over 30 independent runs for four different lambda values: $\lambda = 0.005,0.05,0.1,0.25$ with the same drooping threshold $\delta = 10^{-8}$. b) Average uncertainty at each iteration. c) Number of data points used at each iteration. Shaded region shows $\pm\sigma$} \label{fig:fixed_lambda}
\end{figure}

%\begin{algorithm}[!t]
%\caption{Sequential Copy(\textbf{int} $T$, \textbf{int} $n$, \textbf{float} $\epsilon$, \textbf{float} $\lambda$, \textbf{Classifier} $f_\mathcal{O}$)}
%\label{alg:sequential_final}
%\begin{algorithmic}[1]
%\State $S_0 \gets \{(z_j, f_\mathcal{O}(z_j)) \;|\;  z_j \sim P_\mathcal{Z}, j=1\dots n \}$
%\State $\theta^*_0 \gets \underset{\theta}{\arg\min} \frac{1}{\big|S_0\big|} \sum\limits_{z\in S_0} \rho^2(z,\theta)$
%\For{$i=1$ to $T$}
%    \State $ \Xi \gets S_{i-1}\bigcup\;\{(z_j, f_\mathcal{O}(z_j)) \;|\;  z_j \sim P_\mathcal{Z}, j=1\dots n \}$
%    \State $S_i \gets \{z| z\in \Xi, \rho(z,\theta^*_{i-1})\geq \delta\}$ \Comment{Filter according to uncertainty}
%    \State $\theta^*_i \gets \underset{\theta}{\arg\min} \frac{1}{\big|S_i\big|}\sum\limits_{z\in S_i} \rho^2(z,\theta) + \lambda\cdot \| \theta-\theta_{i-1}^* \|\Bigg |_{\theta^0 = \theta^*_{i-1}}$\Comment{Partial optimization from $\theta^*_{i-1}$}
%\EndFor

%\State \Return $\theta^*_i$

%\end{algorithmic}
%\end{algorithm}

\subsubsection{Automatic Lambda}\label{lab:lambda}

In the previous section, we observed an increment in the number of data points when using large lambda values, to compensate for the memorization effect introduced by the regularization term. This is to consolidate the previous knowledge acquired by the model. In contrast, small lambda values promote short-lived data compression. This is desirable at the beginning of the learning process or when the data distribution suffers a shift. To retain the best of both regimes, we propose a heuristic to dynamically adapt the $\lambda$ value to the needs of the learning process. Our underlying intuition is that, whenever the amount of data required increases, the memory term prevents the copy from adapting to new data. This signals that the value of $\lambda$ must decrease. Equivalently, when we observe a decrement in the number of data points, this means that the model can classify most of them correctly. This indicates that we must stabilize it to avoid unnecessary model drift due to future disturbances in the data. Hence, we must increase the $\lambda$ parameter. 

Thus, considering the set of data at iteration $t$, $S_t$, we force the described behavior by automatically updating the value of $\lambda$ parameter as follows

$$
\lambda=\left\{\begin{tabular}{ll}
$\lambda/2$ & if $\quad |S_t|\geq|S_{t-1}|$,\\ $1.5\cdot\lambda$ & \text{otherwise.}\\
\end{tabular}\right.$$

The updated optimization process, including the modifications discussed, is presented in Algorithm~\ref{alg:sequential_final_lambda}, the final algorithm proposed in this article. Our implementation models data trends by computing the difference between the number of points in the previous iteration and the number of points after data filtering. The filtering process occurs at the beginning of each iteration after generating a new set of $n$ samples.

\begin{algorithm}[!t]
\caption{Sequential approach with alternating optimization}
\label{alg:sequential_final_lambda}
\begin{algorithmic}[1]
\State \textbf{Input:} \textbf{int} $T$, \textbf{int} $n$, \textbf{float} $\varepsilon$, \textbf{Classifier} $f_\mathcal{O}$
\State \textbf{Output}: Optimal copy parameters
\State $S_0 \gets \{(z_j, f_\mathcal{O}(z_j)) \;|\;  z_j \sim P_\mathcal{Z}, j=1\dots n \}$
\State $\theta^*_0 \gets \underset{\theta}{\arg\min} \frac{1}{\big|S_0\big|} \sum\limits_{z\in S_0} \rho^2(z,\theta)$
\For{$t=1$ to $T$}
\State $S_t^* \gets$ Step1($S^*_{t-1}$, $n$, $\delta$, $P_\mathcal{Z}$, $\rho()$, $f_\mathcal{O}$, $f_\mathcal{C}(\theta_{t-1}^*)$)\Comment{Algorithm \ref{alg:filtering}}
 
\State $\lambda \gets\left\{\begin{tabular}{ll}
$\lambda/2$ & if $\quad |S_t|\geq|S_{t-1}|$,\\ $1.5\cdot\lambda$ & \text{otherwise.}\\
\end{tabular}\right.$ \Comment{Adaptive lambda}
   \State $\theta_t^* \gets $Step2($S_t$, $\varepsilon$,$f_\mathcal{O}$, $\theta^*_{t-1}$)\Comment{Algorithm \ref{alg:retraining_regu}}
\EndFor

\State \Return $\theta^*_t$

\end{algorithmic}
\end{algorithm}

As before, we show how this improvement works for the ~\textit{spirals} example. We repeat the experiments using Algorithm~\ref{alg:sequential_final_lambda} with $n=100$ and three different dropping thresholds: $\delta=10^{-6},10^{-8}$ and $10^{-10}$ as we did in Subsection~\ref{lab:filtering} where the dropping procedure was first introduced. Recall that, as discussed in Figure \ref{fig:drooping_no_lambda} only the setting corresponding to the smallest $\delta$ value managed to perform better than the online approach, yet still performed worse than the pure sequential implementation. The remaining thresholds lost too many data points, and their performance was worse than the online approach when the number of data points used was similar. The results obtained when implementing the full Algorithm~\ref{alg:sequential_final_lambda} are depicted in Figure~\ref{fig:automatic_lambda}. The accuracy level obtained using the automatic regularization term is comparable with the desired optimal \textit{pure sequential} case, where the model keeps all data points. Even for the most conservative approach, where we use a very small dropping threshold, the number of points used after 30 iterations is $1/3$ smaller than those required for the pure sequential setting. Moreover, even when deliberately forcing a very volatile setup by using a large value of delta, $\delta=10^{-6}$, the obtained results exhibit an accuracy larger than $0.95$ for an almost constant number of data points equal to $200$.

%TODO \textcolor{red}{FALTA AÑADIR LA CURVA DE SEQUENTIAL PURO}
\begin{figure}[!ht]
  \centering
  \subfloat[]{ \includegraphics[width=0.47\columnwidth]{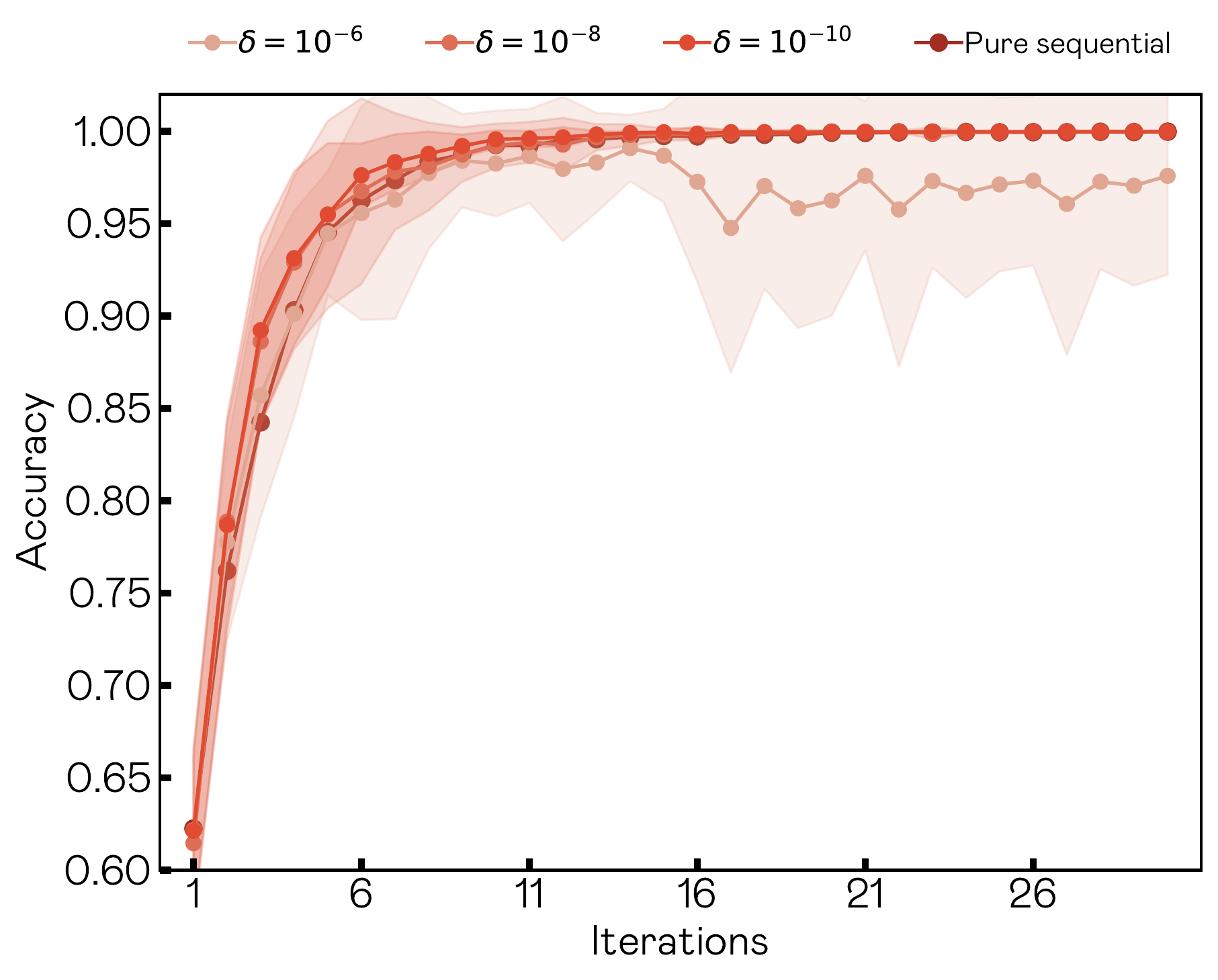}}
  \subfloat[]{ \includegraphics[width=0.47\columnwidth]{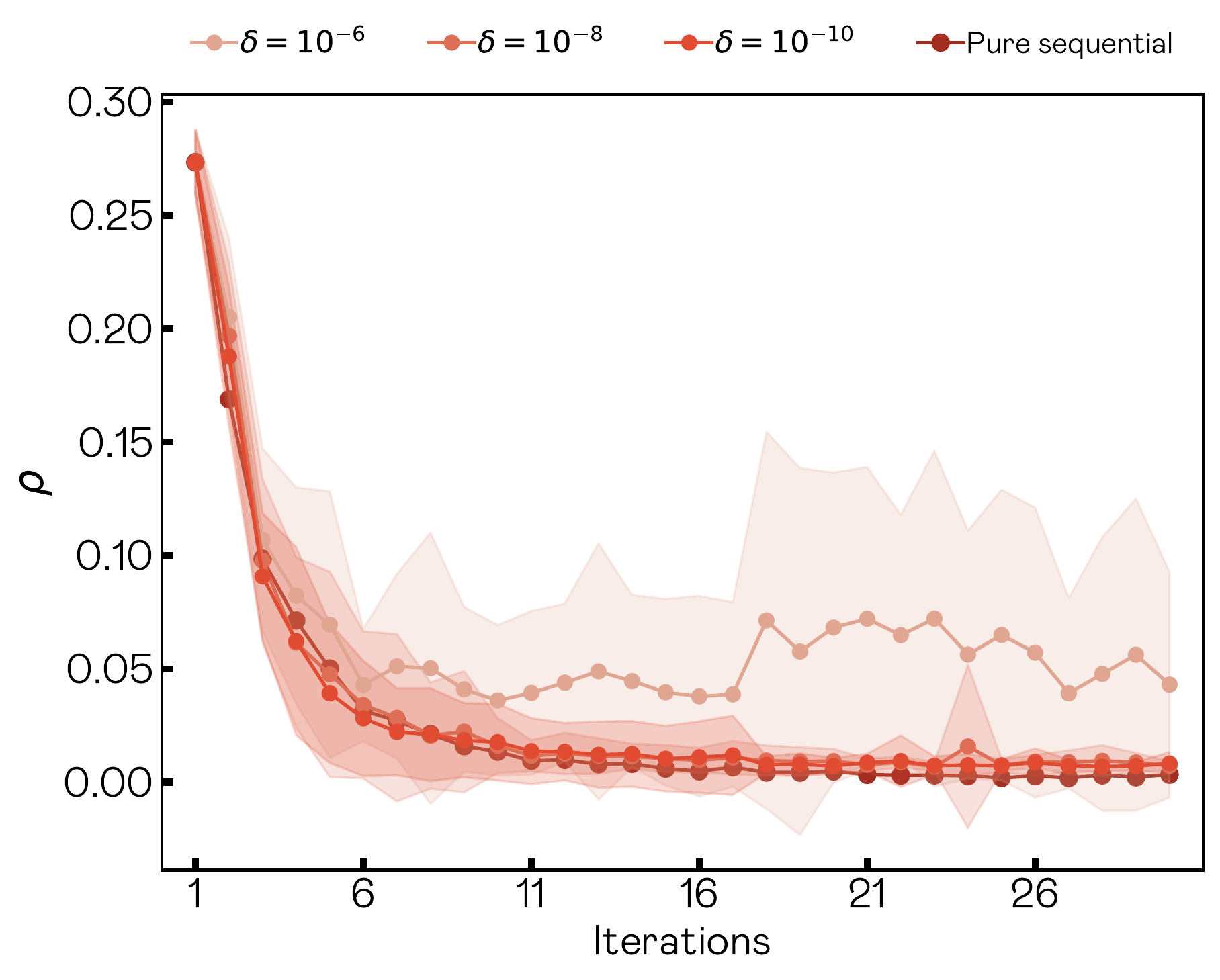}}
  
  \subfloat[]{ \includegraphics[width=0.47\columnwidth]{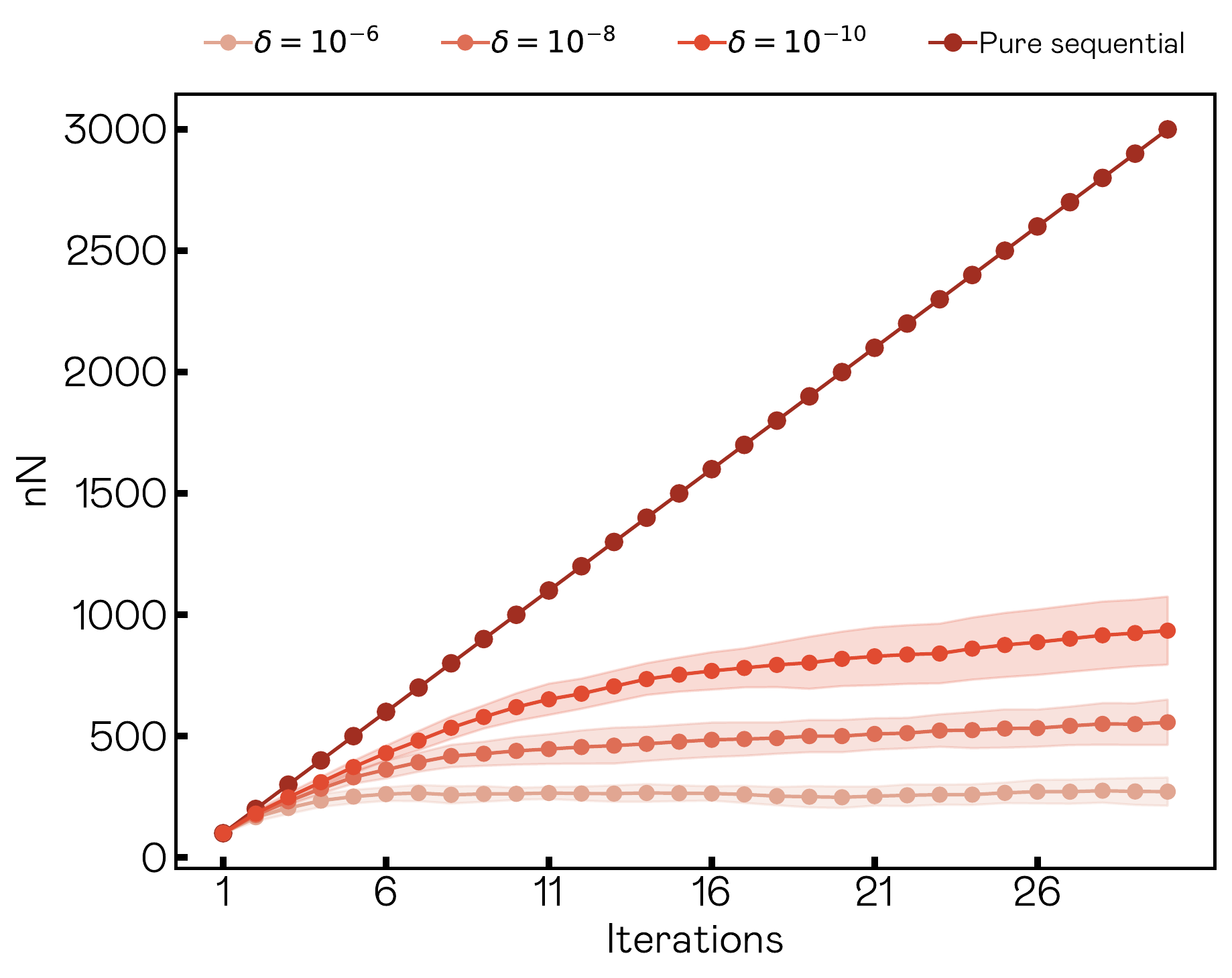}}
  %\subfloat[Number of points per iteration]{ \includegraphics[width=0.47\columnwidth]{figures/sequential/fig5_d.pdf}}
    \caption{a) Copy accuracy averaged over 30 independent runs for three different drooping thresholds: $10^{-6},\, 10^{-8}$, and $10^{-10}$. b) Average uncertainty at each iteration. c) Number of data points used at each iteration. Shaded region shows $\pm\sigma$} \label{fig:automatic_lambda}
\end{figure}

%If we use the automatic lambda and fix the feeding number of points $n$, the only parameter to tune is the dropping threshold $\delta$. The goal when tuning the dropping threshold is to obtain the highest accuracy with the lowest number of data points. Another property is desirable to our results: \textit{speed of convergence}. As previously mentioned in Section~\ref{lab:introducing_sequential_experiments}, the accuracy level in the sequential approach converges faster than the single pass. In the following section, we will introduce and use those metrics with real-world databases.

\section{Experiments}
\label{sec:experiments}
We present the results of our experimental study on the sequential copy approach applied to a set of heterogeneous problems. Our results are analyzed using various performance metrics and compared to the single-pass approach in both one-shot and online settings. Before presenting the results, we provide a clear and reproducible description of the data and experimental setup.

\subsection{Experimental set-up}

{\bf Data.}
We use 58 datasets from the UCI Machine Learning Repository database~\citep{Dheeru2017UCIRepository} and follow the experimental methodology outlined in~\citep{ref:Unceta:2020}. For a more detailed explanation of the problem selection and data preparation process, we refer the reader to the mentioned article.\newline

\noindent
{\bf Pre-processing.}
We convert all nominal attributes to numerical and rescale variables to have zero mean and unit variance. We split the pre-processed data into stratified 80/20 training and test sets. We sort the datasets in alphabetical order and group them in sets of 10, and assign each group one of the following classifiers:  AdaBoost (\textit{\textit{adaboost}}), artificial neural network (\textit{\textit{ann}}), random forest (\textit{rfc}), linear SVM (\textit{linear\_svm}), radial basis kernel SVM (\textit{rbf\_svm}) and gradient-boosted tree (\textit{xgboost}). We train all models using a $3$-fold cross-validation over a fixed parameter grid. A full description of the 58 datasets, including general data attributes, and their assigned classifier, can be found in Table~\ref{tab:UCI_description}.\newline% in Appendix~\ref{Sec:AppD}.

\noindent
{\bf Models.}
We copy the resulting original classifiers using a fully-connected neural network with three hidden layers, each consisting of 64, 32, and 10 \textit{ReLu} neurons and a SoftMax output layer. We use no pre-training or drop-out and initialize the weights randomly. We implement the sequential copy process as described in Algorithm 5 above.\newline

\noindent
{\bf Parameter setting.}
We train these models sequentially for 30 iterations. At each iteration $t$, we generate $n=100$ new data points by randomly sampling a standard normal distribution $\mathcal{N}(0,1)$. We use Algorithm \ref{alg:sequential_final_lambda}, discarding any data points for which the instantaneous copy model has an uncertainty $\rho$ below a defined threshold $\delta$. We adjust the weights at each iteration using the \textit{Adam} optimizer with a learning rate of $5 \cdot 10^{-4}$. For each value of $t$, we use 1000 epochs with balanced batches of 32 data points. We use the previously defined normalized uncertainty average as the loss function and evaluate the impact of the $\delta$ parameter by running independent trials for $\delta \in \{ 5\cdot 10^{-4}, 10^{-4}, 5\cdot 10^{-5}, 10^{-5}, 5\cdot 10^{-6}, 10^{-6}, 5\cdot 10^{-7}, 10^{-7}, 5\cdot 10^{-8}, 10^{-8}, 10^{-9}, 10^{-10}\}$. Additionally, we allow the $\lambda$ parameter to be updated automatically, starting from a value of $0.5$. \newline

\noindent
{\bf Hardware.}
We perform all experiments on a server with 28 dual-core AMD EPYC 7453 at 2.75 GHz, and equipped with 768 Gb RDIMM/3200 of RAM. The server runs on Linux 5.4.0. We implement all experiment in Python 3.10.4 and train copies using TensorFlow 2.8. For validation purposes, we repeat each experiment 30 times and present the average results of all repetitions.

\begin{table}[h!]
\centering
%\begin{sideways}
\tiny
%\small
    \begin{tabular}{@{}lccccc}%cccccccccc} 
    \toprule
    \\[-1em]
    \scriptsize{Dataset} &
    \centering\scriptsize{Classes} & \centering\scriptsize{Samples} & \centering\scriptsize{Features} &
    \centering\scriptsize{Original} & 
    $\mathcal{A}_\mathcal{O}$\\
    %\multicolumn{3}{c}{\multirow{1}{*}%{\centering\scriptsize{\textit{decision\_tree}}}} & \multicolumn{3}{c}{\multirow{1}{*}{\centering\scriptsize{\textit{logistic\_regression}}}} &
    %\multicolumn{3}{c}{\multirow{1}{*}{\centering\scriptsize{\textit{random\_forest}}}}\\
    \\[-0.5em]
    \cline{1-6} %\cline{7-15}
    \\[-0.75em]
    %& & & & & & \multirow{1}{0.4cm}{\centering\scriptsize{$\mathcal{A}_\mathcal{C}$}} & \multirow{1}{0.7cm}{\centering\scriptsize{$\widehat{\mathcal{A}}_\mathcal{C}$}} & \multirow{1}{0.4cm}{\centering\scriptsize{$R_\mathcal{F}^{\mathscr{D}}$}} & \multirow{1}{0.4cm}{\centering\scriptsize{$\mathcal{A}_\mathcal{C}$}} & \multirow{1}{0.7cm}{\centering\scriptsize{$\widehat{\mathcal{A}}_\mathcal{C}$}} & \multirow{1}{0.4cm}{\centering\scriptsize{$R_\mathcal{F}^{\mathscr{D}}$}} &
   % \multirow{1}{0.4cm}{\centering\scriptsize{$\mathcal{A}_\mathcal{C}$}} & \multirow{1}{0.7cm}{\centering\scriptsize{$\widehat{\mathcal{A}}_\mathcal{C}$}} & \multirow{1}{0.4cm}{\centering\scriptsize{$R_\mathcal{F}^{\mathscr{D}}$}} \\
    %\\[-0.5em]
    %\toprule
    %\\[-0.75em]
$\text{abalone}$	&	3	&	4177	&	8	&	$\textit{adaboost}$	&	0.545	\\
$\text{acute-inflammation}$	&	2	&	120	&	6	&	$\textit{adaboost}$	&	1.0	\\
$\text{acute-nephritis}$	&	2	&	120	&	6	&	$\textit{adaboost}$	&	1.0	\\
$\text{bank}$	&	2	&	4521	&	16	&	$\textit{adaboost}$	&	0.872	\\
$\text{breast-cancer-wisc-diag}$	&	2	&	569	&	30	&	$\textit{adaboost}$	&	0.921	\\
$\text{breast-cancer-wisc-prog}$	&	2	&	198	&	33	&	$\textit{adaboost}$	&	0.7	\\
$\text{breast-cancer-wisc}$	&	2	&	699	&	9	&	$\textit{adaboost}$	&	0.914	\\
$\text{breast-cancer}$	&	2	&	286	&	9	&	$\textit{adaboost}$	&	0.69	\\
$\text{breast-tissue}$	&	6	&	106	&	9	&	$\textit{adaboost}$	&	0.545	\\
$\text{chess-krvkp}$	&	2	&	3196	&	36	&	$\textit{ann}$	&	0.995	\\
$\text{congressional-voting}$	&	2	&	435	&	16	&	$\textit{ann}$	&	0.609	\\
$\text{conn-bench-sonar-mines-rocks}$	&	2	&	208	&	60	&	$\textit{ann}$	&	0.833	\\
$\text{connect-4}$	&	2	&	67557	&	42	&	$\textit{ann}$	&	0.875	\\
$\text{contrac}$	&	3	&	1473	&	9	&	$\textit{ann}$	&	0.573	\\
$\text{credit-approval}$	&	2	&	690	&	15	&	$\textit{ann}$	&	0.79	\\
$\text{cylinder-bands}$	&	2	&	512	&	35	&	$\textit{ann}$	&	0.777	\\
$\text{echocardiogram}$	&	2	&	131	&	10	&	$\textit{ann}$	&	0.815	\\
$\text{energy-y1}$	&	3	&	768	&	8	&	$\textit{ann}$	&	0.974	\\
$\text{energy-y2}$	&	3	&	768	&	8	&	$\textit{ann}$	&	0.922	\\
$\text{fertility}$	&	2	&	100	&	9	&	$\textit{random\_forest}$	&	0.9	\\
$\text{haberman-survival}$	&	2	&	306	&	3	&	$\textit{random\_forest}$	&	0.613	\\
$\text{heart-hungarian}$	&	2	&	294	&	12	&	$\textit{random\_forest}$	&	0.763	\\
$\text{hepatitis}$	&	2	&	155	&	19	&	$\textit{random\_forest}$	&	0.742	\\
$\text{ilpd-indian-liver}$	&	2	&	583	&	9	&	$\textit{random\_forest}$	&	0.615	\\
$\text{ionosphere}$	&	2	&	351	&	33	&	$\textit{random\_forest}$	&	0.944	\\
$\text{iris}$	&	3	&	150	&	4	&	$\textit{random\_forest}$	&	0.933	\\
$\text{magic}$	&	2	&	19020	&	10	&	$\textit{linear\_svm}$	&	0.801	\\
$\text{mammographic}$	&	2	&	961	&	5	&	$\textit{random\_forest}$	&	0.803	\\
$\text{miniboone}$	&	2	&	130064	&	50	&	$\textit{random\_forest}$	&	0.936	\\
$\text{molec-biol-splice}$	&	3	&	3190	&	60	&	$\textit{random\_forest}$	&	0.944	\\
%$\text{moons}$	&	2	&	10000	&	2	&	$\textit{rbf\_svm}$	&	0.989	\\
$\text{mushroom}$	&	2	&	8124	&	21	&	$\textit{linear\_svm}$	&	0.979	\\
$\text{musk-1}$	&	2	&	476	&	166	&	$\textit{linear\_svm}$	&	0.812	\\
$\text{musk-2}$	&	2	&	6598	&	166	&	$\textit{linear\_svm}$	&	0.958	\\
$\text{oocytes\_merluccius\_nucleus\_4d}$	&	2	&	1022	&	41	&	$\textit{linear\_svm}$	&	0.771	\\
$\text{oocytes\_trisopterus\_nucleus\_2f}$	&	2	&	912	&	25	&	$\textit{linear\_svm}$	&	0.803	\\
$\text{parkinsons}$	&	2	&	195	&	22	&	$\textit{linear\_svm}$	&	0.923	\\
$\text{pima}$	&	2	&	768	&	8	&	$\textit{linear\_svm}$	&	0.721	\\
$\text{pittsburg-bridges-MATERIAL}$	&	3	&	106	&	7	&	$\textit{linear\_svm}$	&	0.909	\\
$\text{pittsburg-bridges-REL-L}$	&	3	&	103	&	7	&	$\textit{linear\_svm}$	&	0.667	\\
$\text{pittsburg-bridges-T-OR-D}$	&	2	&	102	&	7	&	$\textit{rbf\_svm}$	&	0.857	\\
$\text{planning}$	&	2	&	182	&	12	&	$\textit{rbf\_svm}$	&	0.703	\\
$\text{ringnorm}$	&	2	&	7400	&	18	&	$\textit{rbf\_svm}$	&	0.983	\\
$\text{seeds}$	&	3	&	210	&	7	&	$\textit{rbf\_svm}$	&	0.881	\\
$\text{spambase}$	&	2	&	4601	&	57	&	$\textit{rbf\_svm}$	&	0.926	\\
%$\text{spirals}$	&	2	&	10000	&	2	&	$\textit{rbf\_svm}$	&	1.0	\\
$\text{statlog-australian-credit}$	&	2	&	690	&	14	&	$\textit{rbf\_svm}$	&	0.681	\\
$\text{statlog-german-credit}$	&	2	&	1000	&	24	&	$\textit{rbf\_svm}$	&	0.765	\\
$\text{statlog-heart}$	&	2	&	270	&	13	&	$\textit{rbf\_svm}$	&	0.852	\\
$\text{statlog-image}$	&	7	&	2310	&	18	&	$\textit{rbf\_svm}$	&	0.952	\\
$\text{statlog-vehicle}$	&	4	&	846	&	18	&	$\textit{xgboost}$	&	0.765	\\
$\text{synthetic-control}$	&	6	&	600	&	60	&	$\textit{rbf\_svm}$	&	1.0	\\
$\text{teaching}$	&	3	&	151	&	5	&	$\textit{xgboost}$	&	0.548	\\
$\text{tic-tac-toe}$	&	2	&	958	&	9	&	$\textit{xgboost}$	&	0.974	\\
$\text{titanic}$	&	2	&	2201	&	3	&	$\textit{xgboost}$	&	0.778	\\
$\text{twonorm}$	&	2	&	7400	&	20	&	$\textit{xgboost}$	&	0.976	\\
$\text{vertebral-column-2clases}$	&	2	&	310	&	6	&	$\textit{xgboost}$	&	0.839	\\
$\text{vertebral-column-3clases}$	&	3	&	310	&	6	&	$\textit{xgboost}$	&	0.806	\\
$\text{waveform-noise}$	&	3	&	5000	&	40	&	$\textit{xgboost}$	&	0.843	\\
$\text{waveform}$	&	3	&	5000	&	21	&	$\textit{xgboost}$	&	0.843	\\
$\text{wine}$	&	3	&	178	&	11	&	$\textit{xgboost}$	&	0.944	\\
%$\text{yinyang}$	&	2	&	10000	&	2	&	$\textit{rbf\_svm}$	&	0.995	\\
\\[-1em]
    \bottomrule
    \end{tabular}
\caption{Description of datasets.} \label{tab:UCI_description}
%\end{sideways}
\end{table}

\subsection{Metrics and performance plots}

We evaluate the copy performance using three metrics, including copy accuracy ($\mathcal{A}_\mathcal{C}$). The copy accuracy is calculated as the accuracy of the copy on the original test set\footnote{Computing this value requires the original test data to be known and accessible. We consider this assumption to be reasonable, as it provides a lower bound for all the other metrics.}. Given a model $f$, we define the copy accuracy as the fraction of correct predictions made by this model on a data set of labelled pairs, $\mathcal{D} = {(x_j,y_j)}$, as:

\begin{equation}
\mathcal{A}_\mathcal{C}^f = \frac{1}{N}\sum\limits{j=1}^N\mathbb{I}[f(x_j) == y_j],
\end{equation}
\noindent
where $N$ is the number of samples and $\mathbb{I}[\text{cond}]$ is the indicator function that returns 1 if the condition is true, i.e., when the model predicts the right outcome. We use the definition above to compare the results obtained when using the sequential approach, $\mathcal{A}_\mathcal{C}^{\text{seq}}$, with those obtained when training copies based on the single-pass approach, $\mathcal{A}_\mathcal{C}^{\text{single}}$.

Another performance metric is the area under the normalized convergence accuracy curve ($conv$). This metric measures the convergence speed of the sequential approach and is defined as the area under the curve of the copy accuracy increment per iteration, normalized by the maximum copy accuracy. The $conv$ value is in the range of 0 to 1 and represents the fraction of time required for the system to reach a steady convergence state. For $T$ iterations, the value of $conv$ is defined as follows
\begin{equation}
\mathit{conv} = \frac{1}{T} \frac{\int_0^T \mathcal{A}_\mathcal{C}^{\text{seq}}(t)\; dt}{\underset{t \in [0,T]} {\max}\mathcal{A}_\mathcal{C}^{\text{seq}}(t)},
\end{equation}

\noindent
where $\mathcal{A}_\mathcal{C}^{\text{seq}}(t)$ corresponds to the copy accuracy increment iteration by iteration.

Intuitively, $\mathit{conv}$ metric measures the time required for the system to reach a steady convergence state. Given $T$ iterations of the sequential approach, a convergence speed of $\mathit{conv}$ means that the algorithm reaches the steady state at step $2 (1 - \mathit{conv}) T$. A $\mathit{conv}$ value of $90\%$ tells us that we can reach convergence as fast as $0.2T$. This is, the system requires only 20\% of the allocated time to converge.
    
Finally, we also introduce the efficiency metric, $\mathit{eff}$. This metric evaluates the computational cost of the copying process in terms of the number of synthetic data points used for training. We compute it by comparing the actual number of points used in the sequential approach with the theoretical number of points that would be used if no sample removal policy was applied. We define $\mathit{eff}$ as:
\begin{equation}
\mathit{eff}=1 - \frac{\int_0^T \eta(t)\quad dt}{\int_0^T n \cdot t \quad dt},
\end{equation}

\noindent
where $\eta(t)$ is the number of points used in the sequential approach at each iteration $t$, as shown in Figure~\ref{fig:automatic_lambda}(c). 

The $\mathit{eff}$ metric models the expected number of samples required for copying. This value can be roughly approximated by $(1-\mathit{eff})/2 $. A $\mathit{eff}$ value of $90\%$ indicates that on average, only $5\%$ of the available data points are used in the process. Note that both the pure sequential approach, where no sample removal policy is used, and the single-pass approach have 0 efficiency because they both use all the available data points for training.

We evaluate the performance of the sequential approach by combining various metrics into a single representation. We consider models with copy accuracy within $5\%$ of the single-pass result ($\mathcal{A}_\mathcal{C}^{\text{seq}}/\mathcal{A}_\mathcal{C}^{\text{single}}|{nT}>0.95$). Out of these configurations, we select the ones with the highest copy accuracy ($\mathcal{A}_\mathcal{C}^{\text{seq}}$), best efficiency ($\mathit{eff}$), and fastest convergence ($\mathit{conv}$), referred to as the \textit{Best accuracy}, \textit{Best efficiency}, and \textit{Best convergence}, respectively. To visualize the results, we present four plots: (1) a comparison of the copy accuracy between the single-pass approach and the \textit{Best accuracy} model, (2) a comparison of the copy accuracy and efficiency between the \textit{Best accuracy} and \textit{Best efficiency} results, (3) a comparison of the copy accuracy and convergence rate between the \textit{Best accuracy} and \textit{Best convergence} configurations, and (4) a demonstration of the relationship between convergence and efficiency. 

\begin{figure}[h!]
  \centering
  \subfloat[]{ \includegraphics[width=0.49\columnwidth]{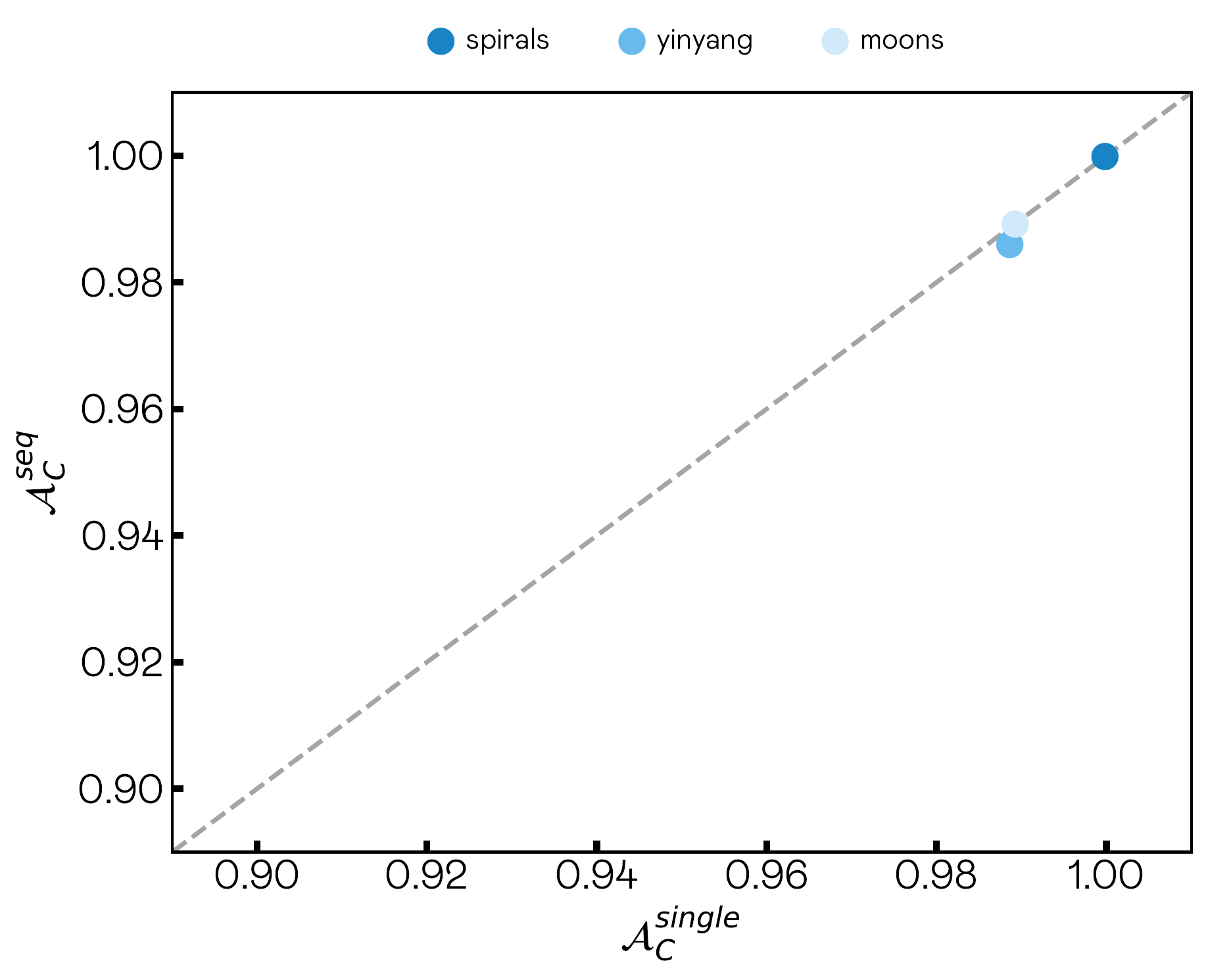}}
  \subfloat[]{ \includegraphics[width=0.49\columnwidth]{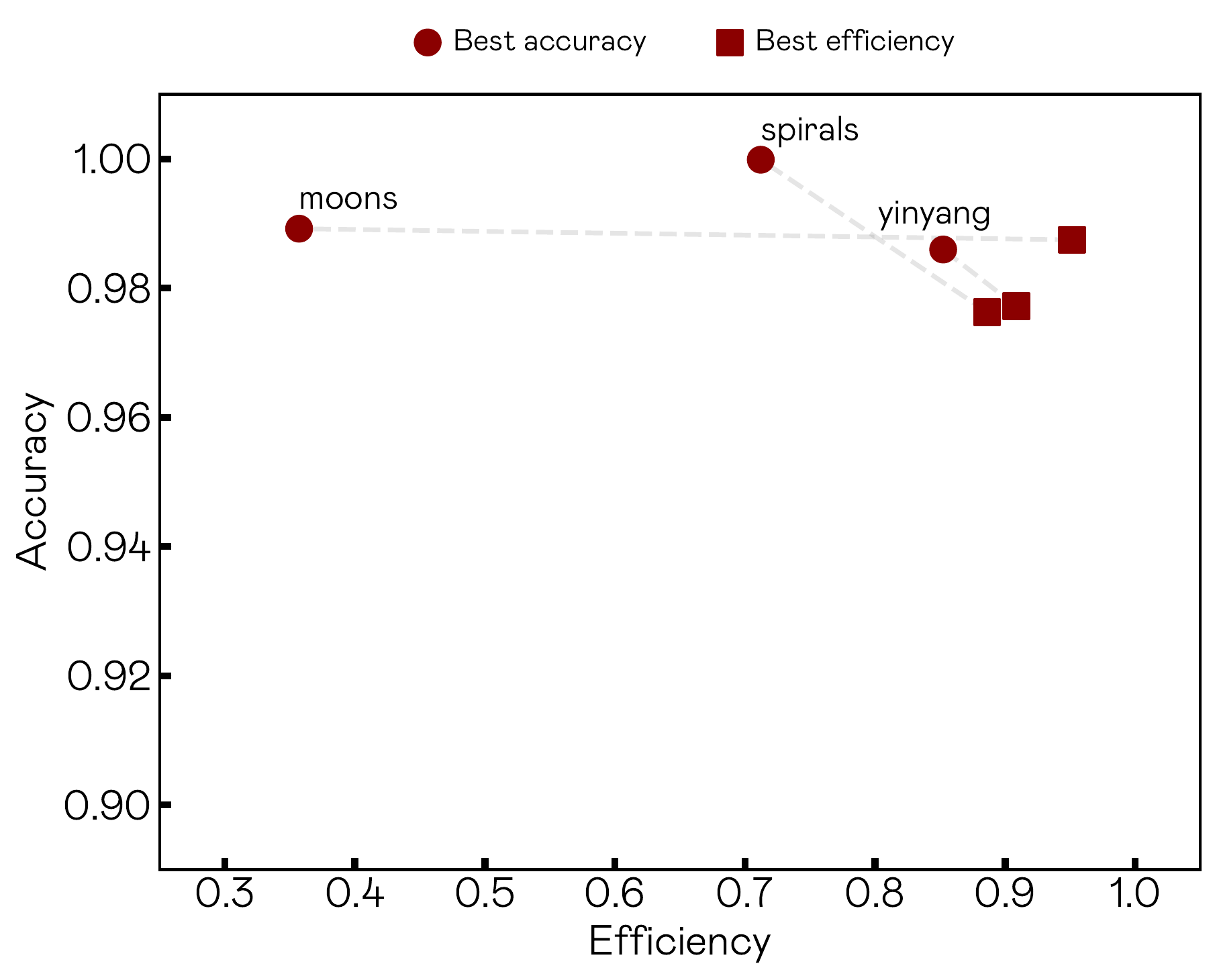}}
  
  \subfloat[]{ \includegraphics[width=0.49\columnwidth]{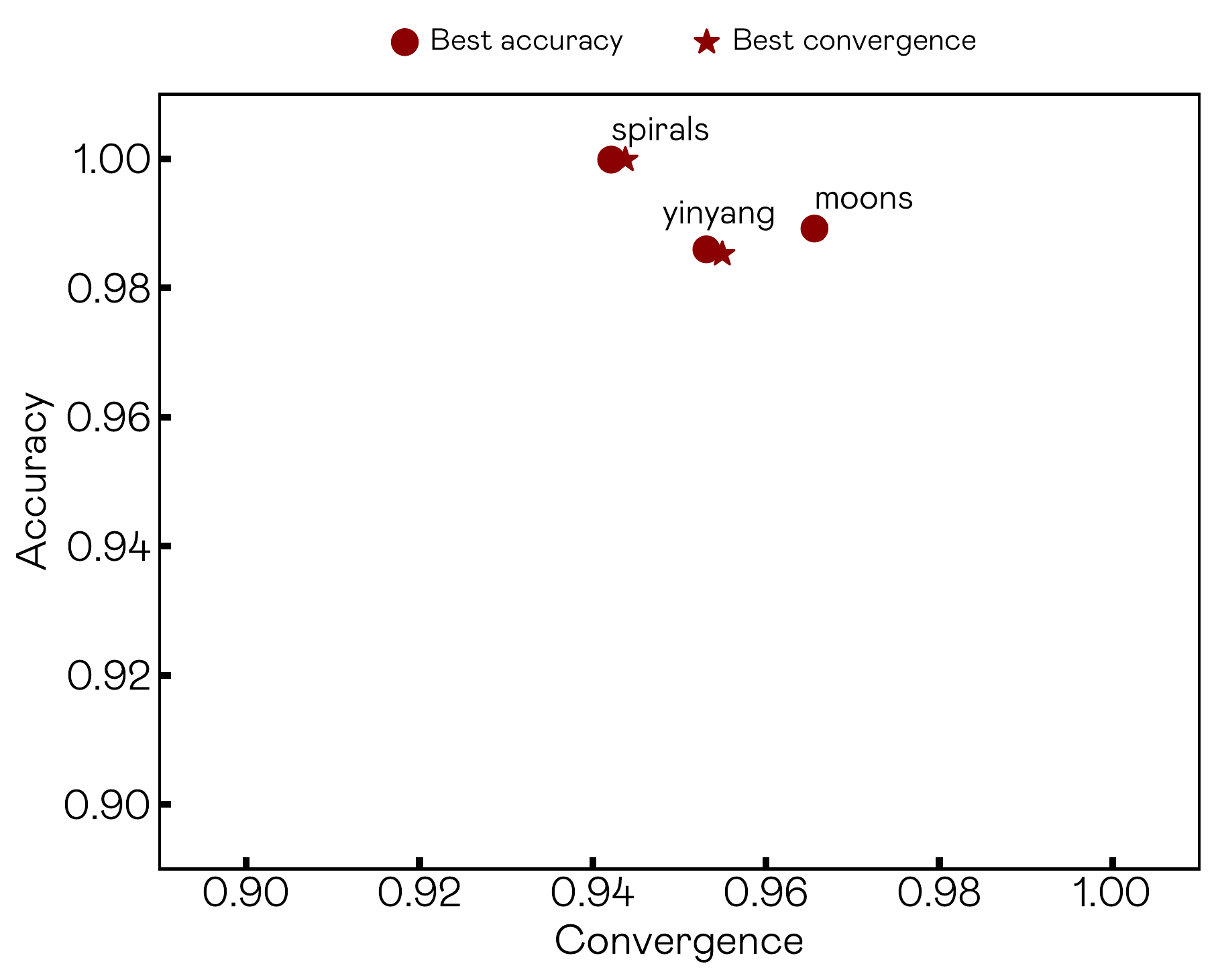}}
  \subfloat[]{ \includegraphics[width=0.49\columnwidth]{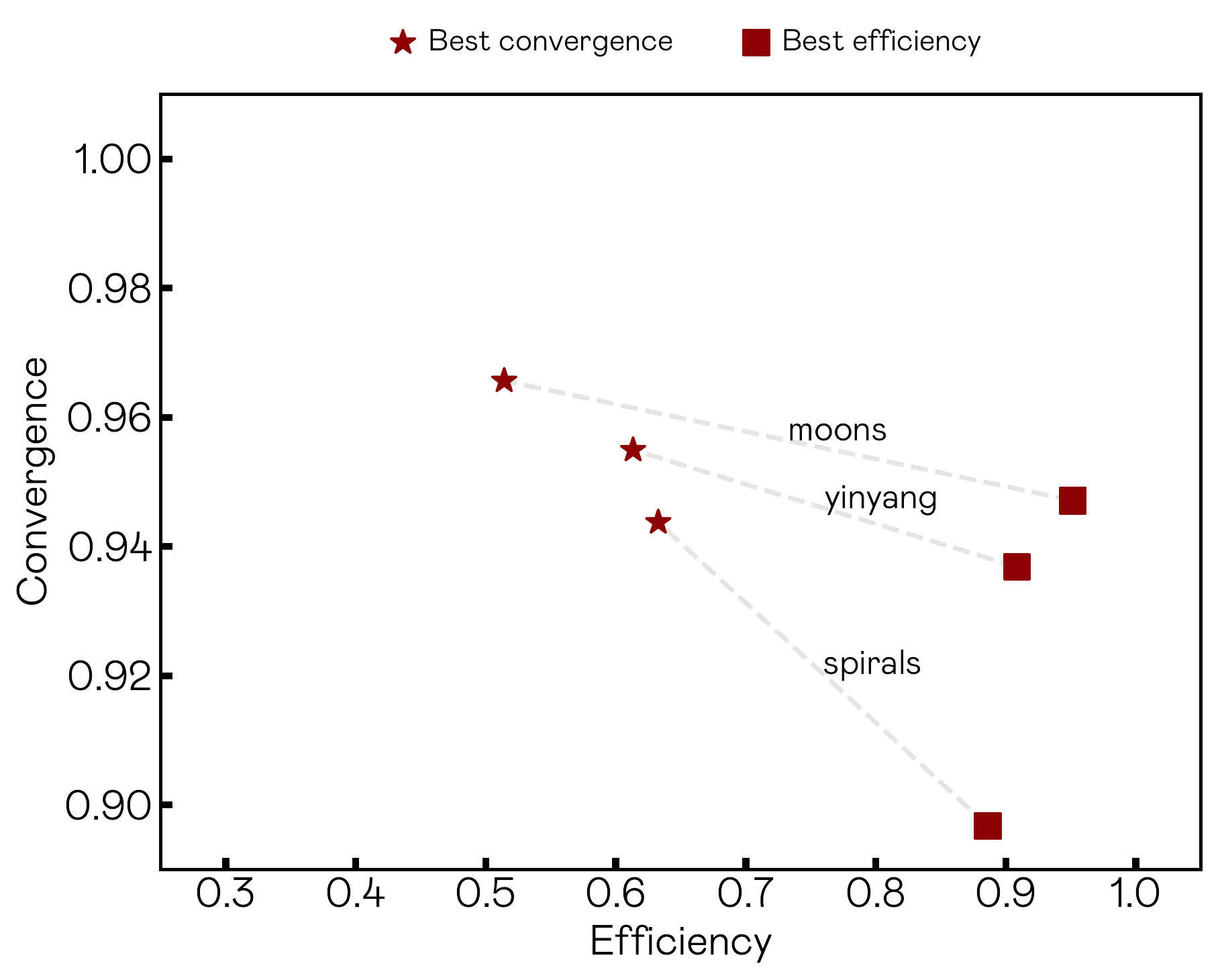}}
  \caption{a) Comparison of accuracy between single-pass and sequential approaches on the \textit{spirals}, \textit{yinyang}, and \textit{moons} datasets. The dashed line marks equal accuracy between the two methods. Points above (below) the line indicate better accuracy for the sequential  (single-pass) approach. b), c), and d) Comparison of \textit{Best accuracy} (circles), \textit{Best efficiency} (squares) and \textit{Best convergence} (stars) operational points for each dataset. The gray dashed lines connect the points of the same dataset and display the linear fit of intermediate solutions.}
\label{fig:toy_combination}
\end{figure}

As an example, Figure~\ref{fig:toy_combination} shows these plots for three toy datasets: \textit{spirals}, \textit{moons}, and \textit{yin-yang}. In Figure~\ref{fig:toy_combination}a), we can compare the performance degradation or improvement of the sequential approach and the single-pass approach. In Figure~\ref{fig:toy_combination}b), a comparison between the gain in efficiency (square marker) and the the best-performing configuration can be seen for each dataset. The most efficient configuration typically displays a significant efficiency gain. Figure~\ref{fig:toy_combination}c) compares the \textit{Best accuracy} and \textit{best convergence} operational points. We observe that both configurations are nearly indistinguishable in terms of both accuracy and convergence, i.e. the most accurate model is also the one which converges faster. Finally, in Figure~\ref{fig:toy_combination}d) we observe that the \textit{Best efficiency} configurations yield a significant improvement in efficiency, while the lose in convergence is not too big.

The biggest gains are therefore obtained in terms of efficiency. This is a relevant results, because it shows that the sequential approach to copying can provide significant advantages for memory usage and computational resource allocation.

\subsection{Results}

We report the metrics for each UCI dataset as introduced above. Table~\ref{tab:results} lists the values of efficiency ($\mathit{eff}$), convergence rate ($\mathit{conv}$), and accuracy $\mathcal{A}_\mathcal{C}^{\text{seq}}$ for the three operational points: {\it Best accuracy}, {\it Best efficiency}, and {\it Best convergence}. These results are compared to the original accuracy ($\mathcal{A}_{O}$) and the one-shot single-pass copy accuracy ($\mathcal{A}_\mathcal{C}^{\text{single}}$).

We observe that not all copies can perfectly reproduce the original accuracy. This effect is observed in 19 datasets for both single-pass and sequential models. Possible reasons include a mismatch between the copy's capacity and the boundary complexity, or a misalignment between the sampling region and the original data distribution. This observation is in line with results in the literature but outside of the scope of this article. Conversely, in 8 datasets, copies outperform the original classifier performance, specifically in 4 cases. This unexpected result may be due to statistical noise and requires further investigation.

The relevant results for our proposal show that the sequential copy process at the {\it Best accuracy} operational point matches the single-pass approach in 53 of the 58 problems, performs worse in 4 datasets and better in 1. On average, the copy process converges in $11.5\%$ of the allotted time with a convergence metric of $\mathit{conv} = 0.942$ and efficiency of $\mathit{eff}=0.716$. This requires an average of $14\%$ of the samples used by the single-pass. A graphic display of the results comparing both approaches is shown in Figure~\ref{fig:uci_plot}. The most notable results have been highlighted in darker colors to ease interpretation. Values plotted in the diagonal correspond to cases where both approaches yield comparable results. Most of the sequential copies recover most of the single-pass accuracy, even when training on smaller synthetic datasets. In some datasets, however, copies based on the sequential approach fall below the diagonal. This effect is observed when the amount of memory ($n$) is not enough to describe the decision boundary entirely. Finally, we observe some cases where points lie above the diagonal, signaling that sequential copies improve over the single-pass results.

The results of the {\it Best efficiency} and {\it Best convergence} operational points have an average accuracy degradation of $5\%$, which is not statistically significant. For the {\it Best efficiency} operational point, 51 datasets have a degradation in performance, but there's a statistically significant improvement in efficiency in 47 datasets (eff=0.882) using only $6\%$ of the data points. The convergence speed increases by $17.4\%$ in the allotted time, meaning 3\% more time compared to the {\it Best accuracy} configuration.

For the {\it Best convergence} operational point, 20 out of 58 datasets show a performance degradation, which is not statistically significant. There's a statistical efficiency loss in 10 datasets, and it requires an average of $15\%$ of the allotted time. This operational point barely improves the convergence speed compared to {\it Best accuracy}. It differs from it in average in 20 datasets, with a required time to reach the steady state of $11.2\%$ of the allotted time. This small time difference is due to the automatic lambda setting. Large $\lambda$ values and small $\delta$ values ensure fast convergence speeds. The automatic lambda algorithm starts with a large lambda value for fast convergence at the initial steps, then reduces it in subsequent iterations to improve accuracy. This results in similar metrics for both {\it Best accuracy} and {\it Best convergence} configurations.

The operational points discussed are graphically represented in Figure~\ref{fig:uci_plot_efficient}. Figures~\ref{fig:uci_plot_efficient}a), b) and c) mark the {\it Best accuracy} operational point with a circle, the {\it Best efficiency} operational point with a square and the {\it Best convergence} operational point with a star. Each dataset is linked by a line connecting the two points, which is a linear fit of the intermediate solutions. Triangles mark the datasets where there is no statistically significant difference between the considered operational points. Figure~\ref{fig:uci_plot_efficient}a) displays that, in most cases, the method's efficiency can be largely increased with only a small degradation in accuracy, as indicated by the relatively flat slopes. Figure~\ref{fig:uci_plot_efficient}b) confirms the results, with the {\it Best accuracy} and {\it Best convergence} operational points tending to be very similar, as evidenced by the high number of triangles. The method thus showcases fast convergence speed and high accuracy simultaneously. Finally, Figure~\ref{fig:uci_plot_efficient}c) displays convergence against efficiency, with the {\it Best convergence} operational point marked with a circle. A dashed line links this point to the {\it Best efficiency} operational point in the same dataset. We observe that there is a larger slope in the lines, which indicates a trade-off between the number of data points used and the speed of convergence. This is to be expected, because the largest the number of points used in the training, the smaller the amount of iterations the system will probably require. However, the method still shows fast convergence, ranging between $85\%$ and $98\%$. This indicates that the method is not only fast but it also requires very small amount of points. 

\begin{landscape}
\begin{table}
\tiny
\centering
    \begin{tabular}{@{}lllllllllllll}
    \toprule
    \\[-1em]
    \multirow{3}{2.5cm}{Dataset} &
    \multirow{3}{0.8cm}{\centering{$\mathcal{A}_\mathcal{O}$}} &
    \multirow{3}{0.8cm}{\centering{$\mathcal{A}_\mathcal{C}^{\text{single}}$}} &
    \multicolumn{3}{c}{\multirow{1}{*}{\centering{Best accuracy}}} & \multicolumn{3}{c}{\multirow{1}{*}{\centering{Best efficiency}}} & \multicolumn{3}{c}{\multirow{1}{*}{\centering{Best Convergence}}}\\
    \\[-0.5em]
    \cline{4-12}
    \\[-0.75em]
    & & & \multirow{1}{0.8cm}{\centering{$\mathit{eff}$}} & \multirow{1}{0.8cm}{\centering{$\mathit{conv}$}} & \multirow{1}{0.8cm}{\centering{$\mathcal{A}_\mathcal{C}^{\text{seq}}$}} & \multirow{1}{0.8cm}{\centering{$\mathit{eff}$}} & \multirow{1}{0.8cm}{\centering{$\mathit{conv}$}} & \multirow{1}{0.8cm}{\centering{$\mathcal{A}_\mathcal{C}^{\text{seq}}$}} & \multirow{1}{0.8cm}{\centering{$\mathit{eff}$}} & \multirow{1}{0.8cm}{\centering{$\mathit{conv}$}} & \multirow{1}{0.8cm}{\centering{$\mathcal{A}_\mathcal{C}^{\text{seq}}$}}  \\
    \\[-0.5em]
    \toprule
    \\[-0.75em]
$\text{abalone}$	&	0.545	&	0.465$\pm$0.095	&	0.085$\pm$0.01	&	0.901$\pm$0.135	&	0.466$\pm$0.087	&	0.623$\pm$0.027	&	0.871$\pm$0.143	&	0.454$\pm$0.087	&	0.378$\pm$0.019	&	0.921$\pm$0.136	&	0.464$\pm$0.079	\\
$\text{acute-inflammation}$	&	1.0	&	1.0$\pm$0.092	&	0.937$\pm$0.032	&	0.966$\pm$0.012	&	1.0$\pm$0.063	&	0.937$\pm$0.032	&	0.966$\pm$0.012	&	1.0$\pm$0.063	&	0.937$\pm$0.032	&	0.966$\pm$0.012	&	1.0$\pm$0.063	\\
$\text{acute-nephritis}$	&	1.0	&	1.0$\pm$0.073	&	0.815$\pm$0.014	&	0.957$\pm$0.017	&	1.0$\pm$0.049	&	0.953$\pm$0.007	&	0.936$\pm$0.037	&	0.983$\pm$0.074	&	0.87$\pm$0.024	&	0.957$\pm$0.019	&	0.996$\pm$0.049	\\
$\text{bank}$	&	0.872	&	0.808$\pm$0.05	&	0.776$\pm$0.041	&	0.88$\pm$0.047	&	0.801$\pm$0.064	&	0.958$\pm$0.005	&	0.845$\pm$0.11	&	0.801$\pm$0.125	&	0.314$\pm$0.023	&	0.956$\pm$0.03	&	0.8$\pm$0.043	\\
$\text{breast-cancer-wisc-diag}$	&	0.921	&	0.904$\pm$0.125	&	0.414$\pm$0.018	&	0.948$\pm$0.07	&	0.851$\pm$0.086	&	0.414$\pm$0.018	&	0.948$\pm$0.07	&	0.851$\pm$0.086	&	0.414$\pm$0.018	&	0.948$\pm$0.07	&	0.851$\pm$0.086	\\
$\text{breast-cancer-wisc-prog}$	&	0.7	&	0.73$\pm$0.069	&	0.976$\pm$0.038	&	0.927$\pm$0.113	&	0.709$\pm$0.117	&	0.976$\pm$0.038	&	0.927$\pm$0.113	&	0.709$\pm$0.117	&	0.976$\pm$0.038	&	0.927$\pm$0.113	&	0.709$\pm$0.117	\\
$\text{breast-cancer-wisc}$	&	0.914	&	0.934$\pm$0.125	&	0.516$\pm$0.019	&	0.967$\pm$0.024	&	0.925$\pm$0.093	&	0.976$\pm$0.003	&	0.844$\pm$0.183	&	0.907$\pm$0.241	&	0.453$\pm$0.02	&	0.97$\pm$0.025	&	0.925$\pm$0.114	\\
$\text{breast-cancer}$	&	0.69	&	0.736$\pm$0.055	&	0.638$\pm$0.02	&	0.944$\pm$0.048	&	0.752$\pm$0.054	&	0.978$\pm$0.019	&	0.919$\pm$0.087	&	0.728$\pm$0.09	&	0.638$\pm$0.02	&	0.944$\pm$0.048	&	0.752$\pm$0.054	\\
$\text{breast-tissue}$	&	0.545	&	0.541$\pm$0.102	&	0.706$\pm$0.015	&	0.93$\pm$0.118	&	0.557$\pm$0.091	&	0.895$\pm$0.014	&	0.915$\pm$0.14	&	0.541$\pm$0.116	&	0.405$\pm$0.024	&	0.938$\pm$0.104	&	0.555$\pm$0.082	\\
$\text{chess-krvkp}$	&	0.995	&	0.983$\pm$0.032	&	0.691$\pm$0.032	&	0.946$\pm$0.014	&	0.961$\pm$0.031	&	0.761$\pm$0.026	&	0.933$\pm$0.017	&	0.941$\pm$0.031	&	0.691$\pm$0.032	&	0.946$\pm$0.014	&	0.961$\pm$0.031	\\
$\text{conn-bench-sonar-mines-rocks}$	&	0.833	&	0.826$\pm$0.097	&	0.695$\pm$0.032	&	0.942$\pm$0.046	&	0.807$\pm$0.101	&	0.765$\pm$0.037	&	0.929$\pm$0.053	&	0.794$\pm$0.101	&	0.695$\pm$0.032	&	0.942$\pm$0.046	&	0.807$\pm$0.101	\\
$\text{connect-4}$	&	0.875	&	0.56$\pm$0.061	&	0.424$\pm$0.025	&	0.952$\pm$0.041	&	0.543$\pm$0.043	&	0.771$\pm$0.039	&	0.934$\pm$0.052	&	0.535$\pm$0.043	&	0.424$\pm$0.025	&	0.952$\pm$0.041	&	0.543$\pm$0.043	\\
$\text{contrac}$	&	0.573	&	0.568$\pm$0.025	&	0.761$\pm$0.012	&	0.964$\pm$0.023	&	0.567$\pm$0.026	&	0.947$\pm$0.009	&	0.943$\pm$0.032	&	0.555$\pm$0.026	&	0.761$\pm$0.012	&	0.964$\pm$0.023	&	0.567$\pm$0.026	\\
$\text{credit-approval}$	&	0.79	&	0.803$\pm$0.032	&	0.925$\pm$0.013	&	0.971$\pm$0.013	&	0.811$\pm$0.017	&	0.988$\pm$0.003	&	0.951$\pm$0.038	&	0.8$\pm$0.072	&	0.917$\pm$0.015	&	0.972$\pm$0.013	&	0.811$\pm$0.02	\\
$\text{cylinder-bands}$	&	0.777	&	0.736$\pm$0.079	&	0.56$\pm$0.022	&	0.936$\pm$0.051	&	0.71$\pm$0.058	&	0.649$\pm$0.032	&	0.921$\pm$0.052	&	0.701$\pm$0.06	&	0.56$\pm$0.022	&	0.936$\pm$0.051	&	0.71$\pm$0.058	\\
$\text{echocardiogram}$	&	0.815	&	0.831$\pm$0.049	&	0.911$\pm$0.009	&	0.964$\pm$0.03	&	0.831$\pm$0.034	&	0.99$\pm$0.013	&	0.963$\pm$0.056	&	0.826$\pm$0.111	&	0.911$\pm$0.009	&	0.964$\pm$0.03	&	0.831$\pm$0.034	\\
$\text{energy-y1}$	&	0.974	&	0.962$\pm$0.046	&	0.621$\pm$0.015	&	0.961$\pm$0.016	&	0.957$\pm$0.038	&	0.937$\pm$0.009	&	0.932$\pm$0.028	&	0.924$\pm$0.048	&	0.621$\pm$0.015	&	0.961$\pm$0.016	&	0.957$\pm$0.038	\\
$\text{energy-y2}$	&	0.922	&	0.899$\pm$0.042	&	0.745$\pm$0.014	&	0.966$\pm$0.019	&	0.905$\pm$0.041	&	0.982$\pm$0.004	&	0.904$\pm$0.054	&	0.854$\pm$0.15	&	0.745$\pm$0.014	&	0.966$\pm$0.019	&	0.905$\pm$0.041	\\
$\text{fertility}$	&	0.9	&	0.91$\pm$0.034	&	0.806$\pm$0.018	&	0.966$\pm$0.032	&	0.908$\pm$0.042	&	0.987$\pm$0.009	&	0.881$\pm$0.195	&	0.88$\pm$0.286	&	0.749$\pm$0.026	&	0.968$\pm$0.028	&	0.908$\pm$0.041	\\
$\text{haberman-survival}$	&	0.613	&	0.651$\pm$0.062	&	0.867$\pm$0.017	&	0.93$\pm$0.061	&	0.648$\pm$0.067	&	0.978$\pm$0.005	&	0.865$\pm$0.159	&	0.623$\pm$0.136	&	0.756$\pm$0.029	&	0.94$\pm$0.047	&	0.647$\pm$0.051	\\
$\text{heart-hungarian}$	&	0.763	&	0.769$\pm$0.057	&	0.639$\pm$0.023	&	0.958$\pm$0.036	&	0.762$\pm$0.048	&	0.964$\pm$0.006	&	0.909$\pm$0.087	&	0.741$\pm$0.094	&	0.576$\pm$0.02	&	0.96$\pm$0.037	&	0.759$\pm$0.054	\\
$\text{hepatitis}$	&	0.742	&	0.811$\pm$0.058	&	0.982$\pm$0.025	&	0.954$\pm$0.101	&	0.813$\pm$0.164	&	0.982$\pm$0.025	&	0.954$\pm$0.101	&	0.813$\pm$0.164	&	0.982$\pm$0.025	&	0.954$\pm$0.101	&	0.813$\pm$0.164	\\
$\text{ilpd-indian-liver}$	&	0.615	&	0.682$\pm$0.049	&	0.535$\pm$0.02	&	0.955$\pm$0.047	&	0.679$\pm$0.049	&	0.979$\pm$0.004	&	0.891$\pm$0.096	&	0.656$\pm$0.087	&	0.535$\pm$0.02	&	0.955$\pm$0.047	&	0.679$\pm$0.049	\\
$\text{ionosphere}$	&	0.944	&	0.932$\pm$0.127	&	0.449$\pm$0.037	&	0.934$\pm$0.051	&	0.914$\pm$0.12	&	0.692$\pm$0.036	&	0.906$\pm$0.062	&	0.892$\pm$0.126	&	0.449$\pm$0.037	&	0.934$\pm$0.051	&	0.914$\pm$0.12	\\
$\text{iris}$	&	0.933	&	0.962$\pm$0.064	&	0.825$\pm$0.013	&	0.967$\pm$0.023	&	0.963$\pm$0.059	&	0.964$\pm$0.005	&	0.931$\pm$0.058	&	0.94$\pm$0.113	&	0.898$\pm$0.014	&	0.967$\pm$0.029	&	0.96$\pm$0.059	\\
$\text{magic}$	&	0.801	&	0.804$\pm$0.022	&	0.777$\pm$0.014	&	0.978$\pm$0.004	&	0.804$\pm$0.01	&	0.991$\pm$0.003	&	0.953$\pm$0.031	&	0.795$\pm$0.096	&	0.777$\pm$0.014	&	0.978$\pm$0.004	&	0.804$\pm$0.01	\\
$\text{mammographic}$	&	0.803	&	0.823$\pm$0.036	&	0.641$\pm$0.023	&	0.962$\pm$0.04	&	0.804$\pm$0.06	&	0.81$\pm$0.02	&	0.942$\pm$0.046	&	0.791$\pm$0.065	&	0.641$\pm$0.023	&	0.962$\pm$0.04	&	0.804$\pm$0.06	\\
$\text{miniboone}$	&	0.936	&	0.736$\pm$0.098	&	0.523$\pm$0.027	&	0.942$\pm$0.103	&	0.699$\pm$0.096	&	0.523$\pm$0.027	&	0.942$\pm$0.103	&	0.699$\pm$0.096	&	0.523$\pm$0.027	&	0.942$\pm$0.103	&	0.699$\pm$0.096	\\
$\text{molec-biol-splice}$	&	0.944	&	0.59$\pm$0.016	&	0.543$\pm$0.031	&	0.936$\pm$0.021	&	0.583$\pm$0.019	&	0.788$\pm$0.031	&	0.921$\pm$0.024	&	0.565$\pm$0.02	&	0.543$\pm$0.031	&	0.936$\pm$0.021	&	0.583$\pm$0.019	\\
$\text{mushroom}$	&	0.979	&	0.962$\pm$0.09	&	0.711$\pm$0.022	&	0.902$\pm$0.05	&	0.93$\pm$0.072	&	0.761$\pm$0.022	&	0.898$\pm$0.052	&	0.919$\pm$0.07	&	0.711$\pm$0.022	&	0.902$\pm$0.05	&	0.93$\pm$0.072	\\
$\text{musk-1}$	&	0.812	&	0.745$\pm$0.065	&	0.773$\pm$0.041	&	0.927$\pm$0.092	&	0.65$\pm$0.09	&	0.773$\pm$0.041	&	0.927$\pm$0.092	&	0.65$\pm$0.09	&	0.773$\pm$0.041	&	0.927$\pm$0.092	&	0.65$\pm$0.09	\\
$\text{musk-2}$	&	0.958	&	0.805$\pm$0.1	&	0.979$\pm$0.003	&	0.838$\pm$0.241	&	0.718$\pm$0.228	&	0.979$\pm$0.003	&	0.838$\pm$0.241	&	0.718$\pm$0.228	&	0.979$\pm$0.003	&	0.838$\pm$0.241	&	0.718$\pm$0.228	\\
$\text{oocytes\_merluccius\_nucleus\_4d}$	&	0.771	&	0.579$\pm$0.096	&	0.728$\pm$0.03	&	0.903$\pm$0.129	&	0.596$\pm$0.105	&	0.985$\pm$0.003	&	0.836$\pm$0.166	&	0.572$\pm$0.138	&	0.728$\pm$0.03	&	0.903$\pm$0.129	&	0.596$\pm$0.105	\\
$\text{oocytes\_trisopterus\_nucleus\_2f}$	&	0.803	&	0.741$\pm$0.061	&	0.714$\pm$0.022	&	0.89$\pm$0.054	&	0.683$\pm$0.054	&	0.763$\pm$0.02	&	0.89$\pm$0.057	&	0.677$\pm$0.051	&	0.714$\pm$0.022	&	0.89$\pm$0.054	&	0.683$\pm$0.054	\\
$\text{parkinsons}$	&	0.923	&	0.878$\pm$0.114	&	0.751$\pm$0.023	&	0.88$\pm$0.089	&	0.86$\pm$0.096	&	0.811$\pm$0.019	&	0.877$\pm$0.081	&	0.837$\pm$0.102	&	0.693$\pm$0.024	&	0.882$\pm$0.087	&	0.851$\pm$0.096	\\
$\text{pima}$	&	0.721	&	0.722$\pm$0.023	&	0.989$\pm$0.003	&	0.968$\pm$0.024	&	0.73$\pm$0.026	&	0.99$\pm$0.004	&	0.958$\pm$0.032	&	0.723$\pm$0.037	&	0.947$\pm$0.008	&	0.969$\pm$0.017	&	0.725$\pm$0.02	\\
$\text{pittsburg-bridges-MATERIAL}$	&	0.909	&	0.909$\pm$0.027	&	0.984$\pm$0.005	&	0.97$\pm$0.021	&	0.911$\pm$0.054	&	0.99$\pm$0.005	&	0.934$\pm$0.076	&	0.898$\pm$0.172	&	0.99$\pm$0.043	&	0.977$\pm$0.039	&	0.909$\pm$0.081	\\
$\text{pittsburg-bridges-REL-L}$	&	0.667	&	0.671$\pm$0.033	&	0.99$\pm$0.003	&	0.963$\pm$0.031	&	0.679$\pm$0.061	&	0.991$\pm$0.004	&	0.95$\pm$0.058	&	0.679$\pm$0.076	&	0.99$\pm$0.003	&	0.963$\pm$0.031	&	0.679$\pm$0.061	\\
$\text{pittsburg-bridges-T-OR-D}$	&	0.857	&	0.857$\pm$0.0	&	0.994$\pm$0.006	&	0.975$\pm$0.002	&	0.862$\pm$0.021	&	0.996$\pm$0.008	&	0.975$\pm$0.001	&	0.86$\pm$0.011	&	0.994$\pm$0.006	&	0.975$\pm$0.002	&	0.862$\pm$0.021	\\
$\text{planning}$	&	0.703	&	0.703$\pm$0.0	&	0.994$\pm$0.025	&	0.98$\pm$0.003	&	0.703$\pm$0.018	&	0.994$\pm$0.025	&	0.98$\pm$0.003	&	0.703$\pm$0.018	&	0.994$\pm$0.025	&	0.98$\pm$0.003	&	0.703$\pm$0.018	\\
$\text{ringnorm}$	&	0.983	&	0.905$\pm$0.037	&	0.687$\pm$0.029	&	0.818$\pm$0.026	&	0.924$\pm$0.035	&	0.912$\pm$0.013	&	0.798$\pm$0.033	&	0.862$\pm$0.04	&	0.687$\pm$0.029	&	0.818$\pm$0.026	&	0.924$\pm$0.035	\\
$\text{seeds}$	&	0.881	&	0.87$\pm$0.113	&	0.85$\pm$0.012	&	0.971$\pm$0.021	&	0.87$\pm$0.069	&	0.984$\pm$0.003	&	0.913$\pm$0.094	&	0.849$\pm$0.13	&	0.76$\pm$0.011	&	0.972$\pm$0.023	&	0.869$\pm$0.069	\\
$\text{spambase}$	&	0.926	&	0.923$\pm$0.032	&	0.741$\pm$0.035	&	0.963$\pm$0.009	&	0.919$\pm$0.024	&	0.983$\pm$0.003	&	0.925$\pm$0.048	&	0.893$\pm$0.119	&	0.741$\pm$0.035	&	0.963$\pm$0.009	&	0.919$\pm$0.024	\\
$\text{statlog-australian-credit}$	&	0.681	&	0.681$\pm$0.0	&	0.995$\pm$0.023	&	0.98$\pm$0.0	&	0.681$\pm$0.0	&	0.995$\pm$0.023	&	0.98$\pm$0.0	&	0.681$\pm$0.0	&	0.995$\pm$0.023	&	0.98$\pm$0.0	&	0.681$\pm$0.0	\\
$\text{statlog-german-credit}$	&	0.765	&	0.723$\pm$0.016	&	0.952$\pm$0.008	&	0.967$\pm$0.019	&	0.73$\pm$0.02	&	0.99$\pm$0.01	&	0.922$\pm$0.075	&	0.708$\pm$0.131	&	0.966$\pm$0.007	&	0.969$\pm$0.021	&	0.73$\pm$0.02	\\
$\text{statlog-heart}$	&	0.852	&	0.848$\pm$0.02	&	0.836$\pm$0.011	&	0.972$\pm$0.013	&	0.849$\pm$0.022	&	0.989$\pm$0.003	&	0.933$\pm$0.036	&	0.822$\pm$0.05	&	0.836$\pm$0.011	&	0.972$\pm$0.013	&	0.849$\pm$0.022	\\
$\text{statlog-image}$	&	0.952	&	0.505$\pm$0.067	&	0.768$\pm$0.022	&	0.851$\pm$0.073	&	0.596$\pm$0.062	&	0.952$\pm$0.005	&	0.735$\pm$0.103	&	0.501$\pm$0.085	&	0.768$\pm$0.022	&	0.851$\pm$0.073	&	0.596$\pm$0.062	\\
$\text{statlog-vehicle}$	&	0.765	&	0.629$\pm$0.068	&	0.291$\pm$0.012	&	0.912$\pm$0.078	&	0.621$\pm$0.07	&	0.415$\pm$0.014	&	0.889$\pm$0.079	&	0.599$\pm$0.069	&	0.291$\pm$0.012	&	0.912$\pm$0.078	&	0.621$\pm$0.07	\\
$\text{synthetic-control}$	&	1.0	&	0.693$\pm$0.083	&	0.878$\pm$0.024	&	0.874$\pm$0.06	&	0.722$\pm$0.08	&	0.878$\pm$0.024	&	0.874$\pm$0.06	&	0.722$\pm$0.08	&	0.853$\pm$0.029	&	0.881$\pm$0.06	&	0.714$\pm$0.072	\\
$\text{teaching}$	&	0.548	&	0.597$\pm$0.095	&	0.296$\pm$0.018	&	0.89$\pm$0.112	&	0.619$\pm$0.095	&	0.602$\pm$0.021	&	0.827$\pm$0.126	&	0.568$\pm$0.106	&	0.296$\pm$0.018	&	0.89$\pm$0.112	&	0.619$\pm$0.095	\\
$\text{tic-tac-toe}$	&	0.974	&	0.876$\pm$0.035	&	0.525$\pm$0.021	&	0.918$\pm$0.023	&	0.891$\pm$0.035	&	0.803$\pm$0.02	&	0.884$\pm$0.03	&	0.835$\pm$0.038	&	0.525$\pm$0.021	&	0.918$\pm$0.023	&	0.891$\pm$0.035	\\
$\text{titanic}$	&	0.778	&	0.778$\pm$0.021	&	0.915$\pm$0.021	&	0.976$\pm$0.009	&	0.774$\pm$0.009	&	0.987$\pm$0.006	&	0.957$\pm$0.085	&	0.769$\pm$0.151	&	0.915$\pm$0.021	&	0.976$\pm$0.009	&	0.774$\pm$0.009	\\
$\text{twonorm}$	&	0.976	&	0.974$\pm$0.029	&	0.59$\pm$0.02	&	0.977$\pm$0.004	&	0.97$\pm$0.013	&	0.982$\pm$0.003	&	0.94$\pm$0.035	&	0.946$\pm$0.062	&	0.59$\pm$0.02	&	0.977$\pm$0.004	&	0.97$\pm$0.013	\\
$\text{vertebral-column-2clases}$	&	0.839	&	0.847$\pm$0.066	&	0.796$\pm$0.014	&	0.962$\pm$0.027	&	0.855$\pm$0.042	&	0.987$\pm$0.004	&	0.891$\pm$0.102	&	0.821$\pm$0.128	&	0.796$\pm$0.014	&	0.962$\pm$0.027	&	0.855$\pm$0.042	\\
$\text{vertebral-column-3clases}$	&	0.806	&	0.825$\pm$0.063	&	0.713$\pm$0.018	&	0.965$\pm$0.039	&	0.819$\pm$0.053	&	0.979$\pm$0.004	&	0.903$\pm$0.066	&	0.799$\pm$0.074	&	0.713$\pm$0.018	&	0.965$\pm$0.039	&	0.819$\pm$0.053	\\
$\text{waveform-noise}$	&	0.843	&	0.841$\pm$0.044	&	0.361$\pm$0.016	&	0.964$\pm$0.013	&	0.825$\pm$0.035	&	0.689$\pm$0.038	&	0.936$\pm$0.016	&	0.802$\pm$0.035	&	0.361$\pm$0.016	&	0.964$\pm$0.013	&	0.825$\pm$0.035	\\
$\text{waveform}$	&	0.843	&	0.833$\pm$0.033	&	0.428$\pm$0.015	&	0.971$\pm$0.013	&	0.825$\pm$0.028	&	0.849$\pm$0.026	&	0.932$\pm$0.023	&	0.796$\pm$0.028	&	0.428$\pm$0.015	&	0.971$\pm$0.013	&	0.825$\pm$0.028	\\
$\text{wine}$	&	0.944	&	0.943$\pm$0.098	&	0.743$\pm$0.016	&	0.962$\pm$0.036	&	0.926$\pm$0.084	&	0.918$\pm$0.014	&	0.924$\pm$0.055	&	0.9$\pm$0.084	&	0.811$\pm$0.02	&	0.965$\pm$0.039	&	0.925$\pm$0.084	\\
\\[-0.25em]
$\textit{average}$	&	0.839 & 0.800 & 0.716 & 0.942 & 0.794 & 0.882 & 0.913 & 0.776 & 0.701 & 0.944 & 0.793	\\
    \bottomrule
    \end{tabular}
\caption{Results of the copying process for the single-pass and sequential approaches.}
\label{tab:results}
\end{table}
\end{landscape}

%\begin{figure}[h!]
%  \centering
%  \subfloat[Max. accuracy Vs. most efficient: UCI datasets]{ \includegraphics[width=0.65\columnwidth]{figures/experiments/metrics_uci.pdf}}
%    \caption{Highest accuracy (blue points) and most efficient (orange squares) results for each dataset. Green-dashed segments connect both points for the same dataset showing a linear fitting of all the intermediate solutions. Red triangles represent datasets where the most efficient point is also the one with highest accuracy.}
%\label{fig:uci_plot_efficient}
%\end{figure}

%\newpage
%\begin{figure}[h!]
%  \centering
%  \subfloat[]{ %\includegraphics[width=0.95\columnwidth]{figures/experiments/comparision_accs_original.pdf}}
 %   \caption{Accuracy comparison between sequential approach and the original model. Points over the dashed line exhibit the same accuracy using both methods whereas points above (below) the line are points where sequential (original model) method gives a better performance.}
%\label{fig:uci_plot}
%\end{figure}

\section{Conclusions}
\label{sec:conclusions}
In this paper, we proposed a sequential approach for replicating the decision behavior of a machine learning model through copying. Our approach offers a unique solution to the problem of balancing memory requirements and convergence speed and is the first to tackle the problem modeled in Equation~\ref{eq:capacity} in the context of copying. To this aim, we moved the copying problem to a probabilistic setting and introduced two theorems to demonstrate that the sequential approach converges to the single-pass approach when the number of samples used for copying increases.

We also studied the duality of compression and memorization in the copy model and showed that a perfect copy can compress all the data in the model parameters. To evaluate this effect, we used epistemic uncertainty as a reliable data compression measure for copying. This measure is only valid for copies and can not be extrapolated to standard learning procedures. This is because contrary to the standard learning case, there is no aleatoric uncertainty when copying. Therefore, all uncertainty measured corresponds to that coming from the model itself. As such, we can devise copy models that effectively compress all the data with guarantees. With this in mind, we have identified the phenomenon of catastrophic forgetting in copies, a well-known effect that appears in online learning processes. To mitigate this effect, we have introduced a regularization term derived from an invariant in one of the theorems that enable the process to become more stable. 

To reduce the computational time and memory used, we also introduced a sample selection policy. This policy controls the compression level that each data sample undertakes. If new data points are already well compressed and represented by the copy at the considered iteration, it is unnecessary to feed them back to this model. As a result of this process, very little data is required during the learning process. Moreover, we observed that the number of samples required to converge to an optimal solution stabilizes to a certain amount.

\begin{figure}[h!]
  \centering
  \includegraphics[width=0.75\columnwidth]{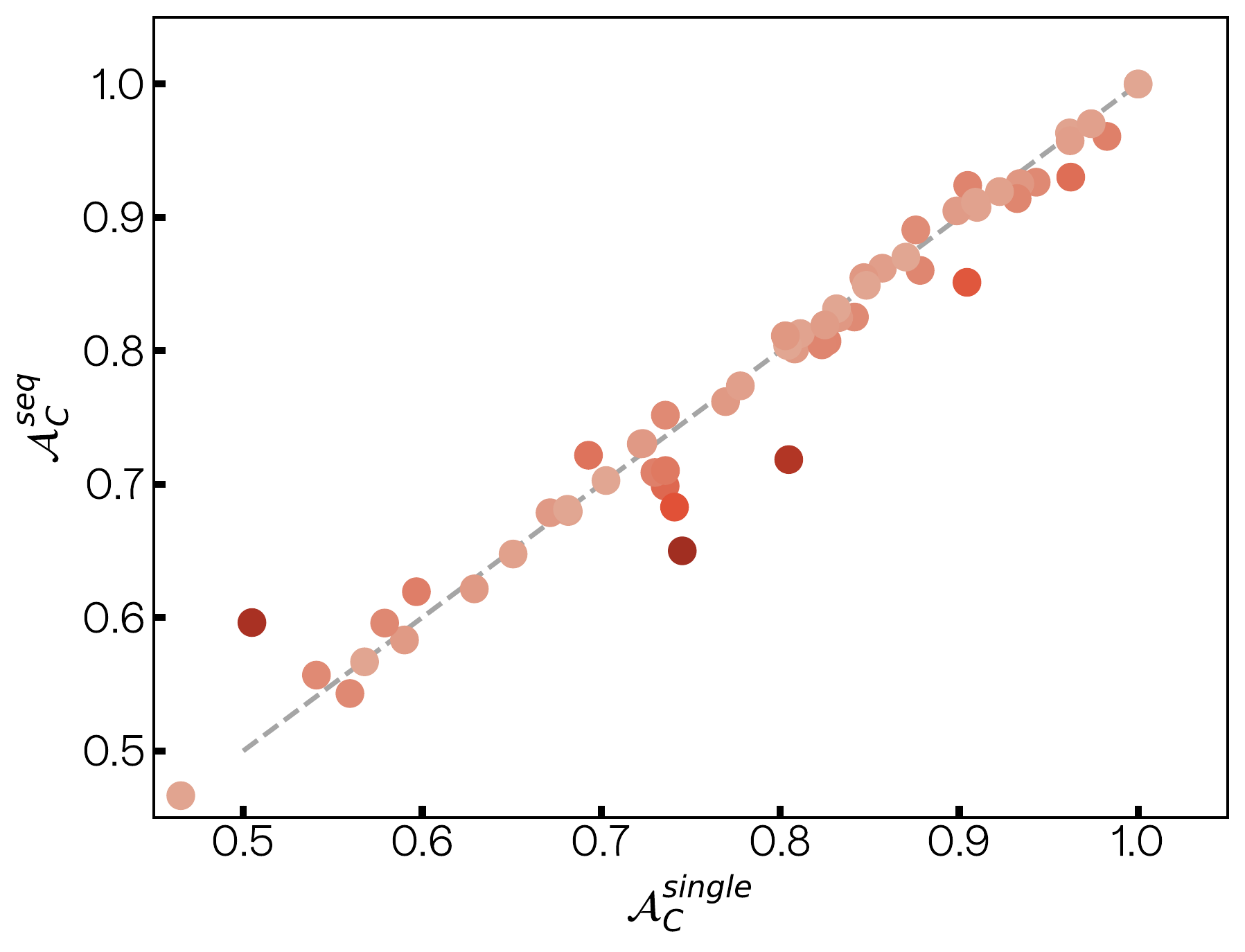}
    \caption{Comparison of accuracy between the one-shot single-pass and sequential approaches. Points over the dashed line exhibit the same accuracy using both methods whereas the points above (below) the line are points where the sequential (single-pass) method gives better performance. Color gets darker with the distance of each point to the dashed line.}
\label{fig:uci_plot}
\end{figure}

Additionally, we introduced a regularization term for the copy-loss function to prevent noise during the learning process. This term prevents copies from diverging from one iteration to another. To control the hyper-parameter governing this regularization term, we devised an automatic adjustment policy. This policy resorts to a simple but stable meta-learning algorithm that allows us to weigh the dynamic adjustment of the regularization term. As a result, there is no need for hyperparameter tuning.

Our empirical validation on 58 UCI datasets and six different machine learning architectures showed that the sequential approach can create a copy with the same accuracy as the single-pass approach while offering faster convergence and more efficient use of computational resources. The sequential approach provides a flexible solution for companies to reduce the maintenance costs of machine learning models in production by choosing the most suitable copying setting based on available computational resources for memory and execution time.

\begin{figure}[h!]
  \centering
  \subfloat[]{ \includegraphics[width=0.5\columnwidth]{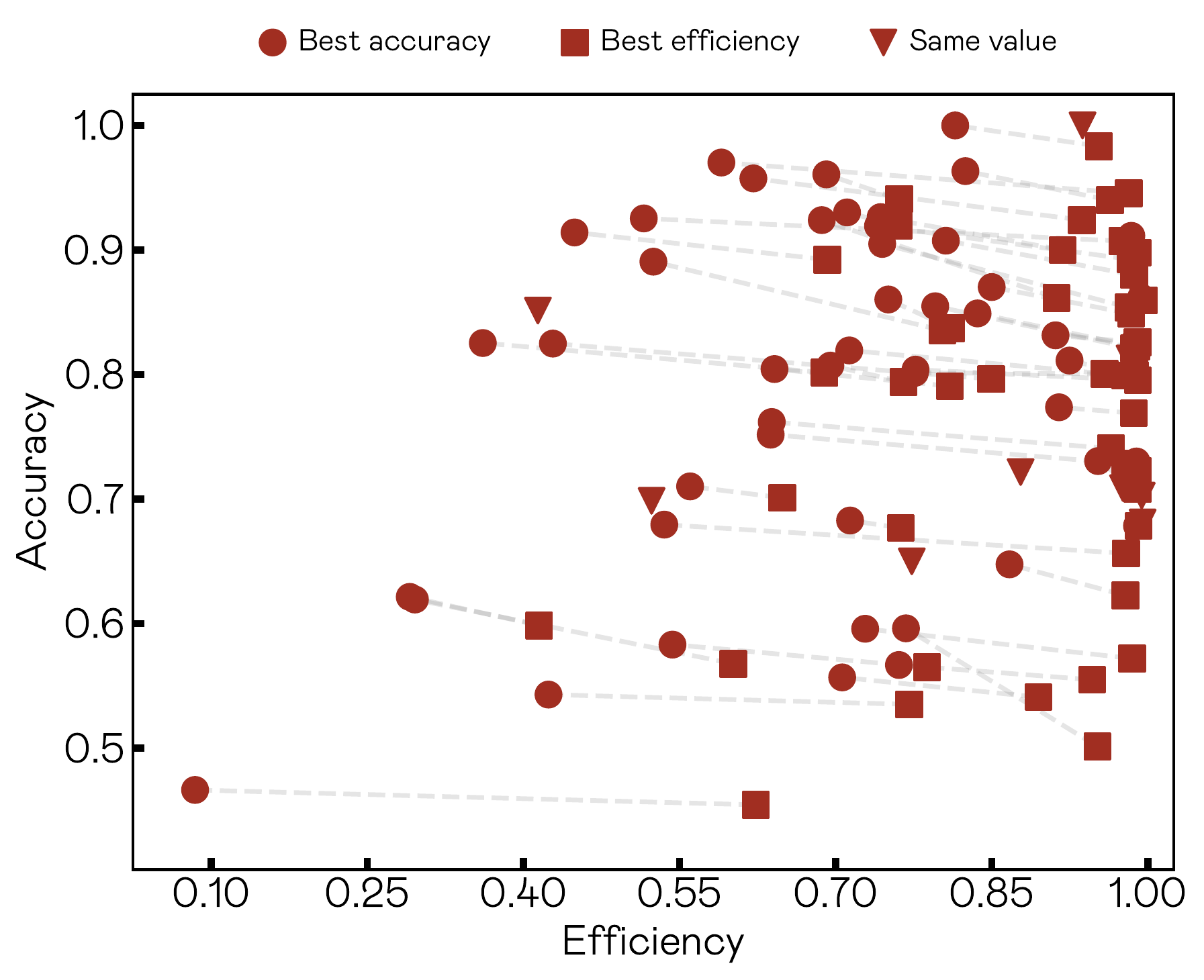}}
  \subfloat[]{ \includegraphics[width=0.5\columnwidth]{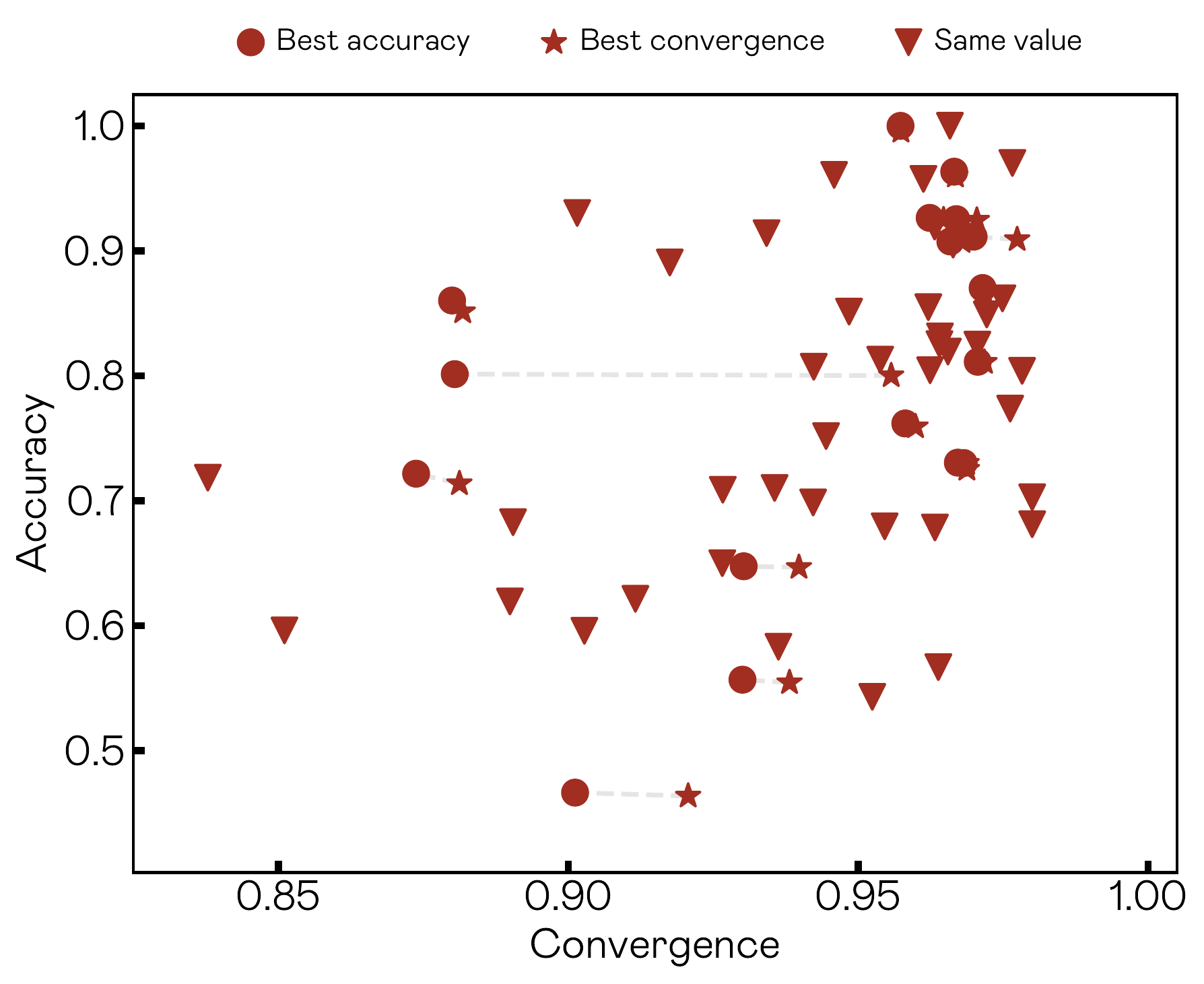}}

  \subfloat[]{ \includegraphics[width=0.5\columnwidth]{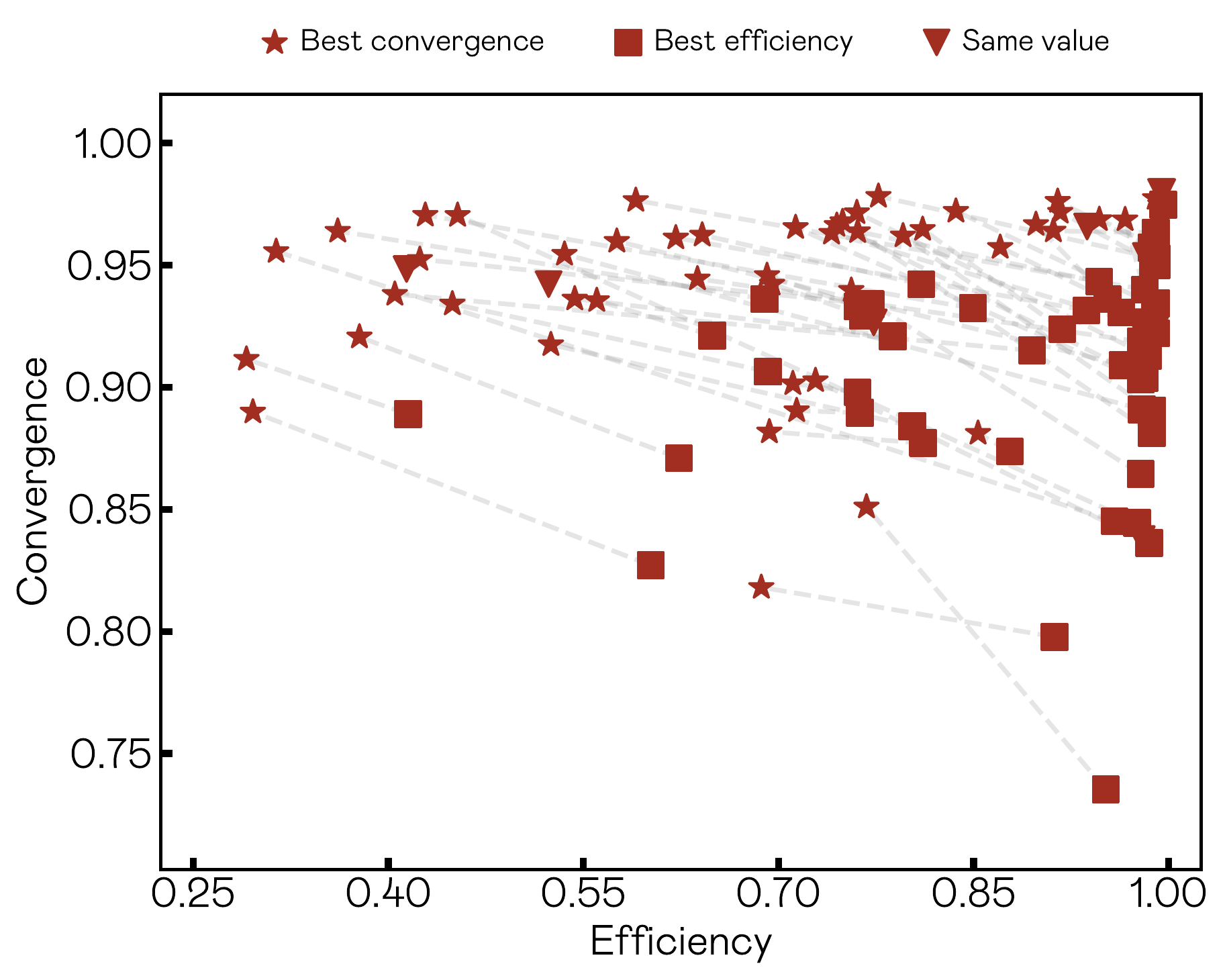}}
    \caption{Comparison of \textit{Best accuracy} (circles), \textit{Best efficiency} (squares) and \textit{Best convergence} (stars) operational points for each dataset in terms of a) accuracy and efficiency, b) accuracy and convergence and c) convergence and efficiency. The inverted triangles identify those cases where the considered operational points are the same. The gray dashed lines connect the points of the same dataset and display the linear fit of intermediate solutions.}
\label{fig:uci_plot_efficient}
\end{figure}

% Acknowledgements should go at the end, before appendices and references

\acks{This work was funded by Huawei Technologies Duesseldorf GmbH under project TC20210924032, and partially supported by MCIN/AEI/10.13039/501100011033 under project PID2019-
105093GB-I00}

%\bibliography{biblio}
%%%%%%%%%%%%%%%%%%%%%%%%%%%%%%%%%%%%%%%%%%%%%%%%%%%%%%%%%%%%%%%%%%%%%%%%%%%%%%%

%%%%%%%%%%%%%%%%%%%%%%%%%%%%%%%%%%%%%%%%%%%%%%%%%%%%%%%%%%%%%%%%%%%%%%%%%%

\newpage
\appendix
\section{Proof of convergence}
\label{Sec:AppA}
Hence, in what follows we show that solving the optimization problem in Equation~\ref{eq:opt} equals solving the same problem in the limit where $i$ approaches infinity. To do this, we first prove that the sequence of functions $F_i$ converges to $F(\theta)$ for increasing values of $i$, {\em i.e.} that $\lim_{i\to\infty}F_i(\theta) = F(\theta)$. Then, we also prove that from this result it derives that the sequence of $\theta^*_i$ converges to $\theta^*$ as $i$ approaches infinity. We do so under several assumptions.

\subsection*{Uniform convergence for $F_i$}

Let us introduce the following proposition on the uniform convergence of functions.
 
\begin{proposition}
A sequence of functions $
    f_i:D\subseteq\mathbb{R}^n\longrightarrow \mathbb{R}
$
is uniformly convergent to a limit function $   f:D\subseteq\mathbb{R}^n\longrightarrow \mathbb{R}
$, if and only if
\begin{equation*}
    ||f-f_i||_{\infty}\xrightarrow[]{i\to\infty} 0
    \label{eq:teoremaUniforme_norma}
\end{equation*}
where $||\cdot||_\infty$ denotes the supremum norm of the functions $f_i-f$ on $D$~\footnote{https://www.bookofproofs.org/branches/supremum-norm-and-uniform-convergence/proof/}.
\end{proposition}

On this basis, let's prove Theorem \ref{theorem:function_convergence} from main text:\\
\begin{theorem}\ref{theorem:function_convergence}
Let $S_i\subseteq S_{i+1}\subseteq\cdots\subseteq S$ be a subsets' convergent sequence. Then, a sequence of functions $\big\{F_i\big\}_i$ defined as $F_i(\theta)=\sum_{z\in S_i} \mathcal{P}(\theta|\fc(z,\theta),\fo(z))$, uniformly converges to $F(\theta)=\sum_{z\in S} \mathcal{P}(\theta|\fc(z,\theta),\fo(z))$
\end{theorem}
\begin{proof}
Let us start by taking the supremum norm
\begin{align*}
    ||F(\theta)-F_i(\theta)||_{\infty} &= \sup_{\theta \in \Theta}\Big\{\Big|F(\theta)-F_i(\theta)\Big|\Big\}\nonumber\\
    &=\sup_{\theta \in \Theta}\Big\{\Big|\sum_{z\in S} \mathcal{P}(\theta|fc(z,\theta),\fo(z))-\sum_{z\in S_i} \mathcal{P}(\theta|fc(z,\theta),\fo(z))\Big|\Big\}\nonumber\\
    &=\sup_{\theta \in \Theta}\Big\{\Big|\sum_{z\in S\setminus S_i}\mathcal{P}(\theta|fc(z,\theta),\fo(z))\Big| \Big\}
\end{align*}

Since $\mathcal{P}(\theta|fc(z,\theta),\fo(z))$ is a probability function, it holds that $0\leq\mathcal{P}(\theta|fc(z,\theta),\fo(z))\leq 1$. Hence, the norm is bounded by the equation below
\begin{align*}
    ||F(\theta)-F_i(\theta)||_{\infty} &=\sup_{\theta \in \Theta}\Big\{\Big|\sum_{z\in S\setminus S_i}\mathcal{P}(\theta|fc(z,\theta),\fo(z))\Big| \Big\}\nonumber\\
    &\leq\sup_{\theta \in \Theta}\Big\{\sum_{z\in S\setminus S_i}\Big|\mathcal{P}(\theta|fc(z,\theta),\fo(z))\Big| \Big\}\leq \Big|S\setminus S_i\Big|
\end{align*}
for $\Big|S\setminus S_i\Big|$ the cardinality of the set $S\setminus S_i$. As discussed before, by definition $S_i$ converges to $S$ for large values of $i$. According to the proposition above, therefore, we can prove that 

\begin{equation*}
    ||F(\theta)-F_i(\theta)||_{\infty} \xrightarrow[]{i\to\infty} 0
\end{equation*}

From this proof, it follows that  when $i$ approaches infinity the function $F_i(\theta)$ uniformly converges to $F(\theta)$

\begin{equation}
    ||F(\theta)-F_i(\theta)||_{\infty} \xrightarrow[]{i\to\infty} 0 \; \Longrightarrow \; F_i(\theta)\rightrightarrows F(\theta)
\end{equation}
\end{proof}

As a consequence of this uniform convergence, two properties of the function $F$ naturally arise. Firstly, $F_i(\theta)$ converges point-wise to $F(\theta)$. Secondly, $F(\theta)$ is a continuous function on $\Theta$.

\subsection*{Parameter convergence}

\begin*{\textbf{Theorem  \ref{theorem:params_convergence}}}
Under the conditions of Theorem~\ref{theorem:function_convergence}, a sequence of parameters $\big\{\theta_i^*\big\}_i$ defined as $\theta_i^*=\argmax_{\theta \in \Theta} F_i(\theta)$, converges to $\theta^*=\argmax_{\theta \in \Theta} F(\theta)$, where $\Theta$ is the complete set of parameter
\end*{.}
\begin{proof}
As previously introduced, we can define the optimal copy parameters for a given value of $i$, and its corresponding function $F_i$, according to the following equation
\begin{equation}
    \theta_i^* = \argmax_{\theta\in\Theta} F_i(\theta)
\end{equation}
for $\Theta$ the complete parameter set. We assume that this set is only \textit{well defined} if $F_i$ and $F$ have a unique global maximum. This is a commonly made assumption in the literature. In addition, we also assume that $\Theta$ is compact. 

From the definition of $\theta_i^*$ above it follows that
\begin{equation*}
    F_i(\theta_i^*)\geq F_i(\theta),\;\; \forall\theta\in\Theta,\;\forall i\in\mathbb{N}
\end{equation*}
and using the point-wise convergence of $F_i$ we obtain 
\begin{equation}
    F(\hat{\theta}^*)\geq F(\theta),\;\; \forall\theta\in\Theta
\end{equation}
where $\hat{\theta}^*$ is the limit of the sequence $\big(\theta_i^*\big)_i$. As a consequence of the compactness of $\Theta$ and the continuity of $F$, we can conclude that $\hat{\theta}^*$ must exist. Moreover, given our assumption that $\Theta$ is \textit{well defined}, we can conclude that
\begin{equation}
    \hat{\theta}^*={\theta}^* =\argmax_\theta F(\theta).
\end{equation}
\end{proof}

The proof above shows that we can approximate the true optimal parameters by sequentially estimating their value using a sequence of subsets $S_i$ that uniformly converge to $S$. In other words, it demonstrates the feasibility of the sequential approach. This is a theoretical feasibility.

\end{document}